\documentclass[sn-mathphys]{sn-jnl}

\usepackage{natbib}
\usepackage{graphicx}%
\usepackage{multirow}%
\usepackage{amsmath,amssymb,amsfonts}%
\usepackage{amsthm}%
\usepackage{mathrsfs}%
\usepackage[title]{appendix}%
\usepackage{xcolor}%
\usepackage{textcomp}%
\usepackage{manyfoot}%
\usepackage{booktabs}%
\usepackage{algorithm}%
\usepackage{algorithmicx}%
\usepackage{algpseudocode}%
\usepackage{listings}%
\usepackage[caption=false,font=normalsize,labelfont=sf,textfont=sf]{subfig}

\theoremstyle{thmstyleone}%
\newtheorem{theorem}{Theorem}
\newtheorem{proposition}[theorem]{Proposition}%

\theoremstyle{thmstyletwo}%

\theoremstyle{thmstylethree}%
\newtheorem{definition}{Definition}%
\algnewcommand\PROCEDURE{\item[\algorithmicprocedure]}%
\algnewcommand\algorithmicendprocedure{\textbf{end procedure}}
\algnewcommand\ENDPROCEDURE{\item[\algorithmicendprocedure]}%
\begin{document}

\title[Article Title]{Learning Generalization and Regularization of Nonhomogeneous Temporal Poisson Processes}

\author[1]{\fnm{Son} \sur{Nguyen Van}}\email{son.nvncs17020@sis.hust.edu.vn}

\author*[2]{\fnm{Hoai} \sur{Nguyen Xuan}}\email{nx.hoai@hutech.edu.vn}
\equalcont{These authors contributed equally to this work.}

\affil[1]{\orgdiv{Computer Science Department}, \orgname{Hanoi University of Science and Technology}, \orgaddress{\street{No. 1 Dai Co Viet road}, \city{Hanoi}, \postcode{10000}, \country{Vietnam}}}

\affil*[2]{\orgdiv{Faculty of Information Technology}, \orgname{HUTECH University}, \orgaddress{\street{No. 475A Dien Bien Phu}, \city{Ho Chi Minh City}, \postcode{72308}, \country{Vietnam}}}


\abstract{The Poisson process, especially the nonhomogeneous Poisson process (NHPP), is an essentially important counting process with numerous real-world applications. Up to date, almost all works in the literature have been on the estimation of NHPPs with infinite data using non-data driven binning methods. In this paper, we formulate the problem of estimation of NHPPs from finite and limited data as a learning generalization problem. We mathematically show that while binning methods are essential for the estimation of NHPPs, they pose a threat of overfitting when the amount of data is limited. We propose a framework for regularized learning of NHPPs with two new adaptive and data driven binning methods that help to remove the ad-hoc tuning of binning parameters. Our methods are experimentally tested on synthetic and real-world datasets and the results show their effectiveness.}




\maketitle

\textbf{This work has been submitted to the IEEE for possible publication. Copyright may be transferred without notice, after which this version may no longer be accessible.}

\section{Introduction}\label{sec:introduction}
Stochastic processes allow us to model uncertain quantities that evolve over time. One of the most fundamental and important of such processes is the Poisson process
\citep{Komaki2021, Jahani2021, Moore2019, Alizadeh2013} as it is appropriate for modelling a large number of completely random phenomena. Therefore, the Poisson process has been widely used in various real-world applications such as taxi routing \citep{SonNV}, queueing systems \citep{bibi-Brown2005}, and e-mail arrival counting \citep{Alizadeh2008, Alizadeh2013}. 

Because of the importance and popularity of Poisson processes, estimation of them from data has been extensively studied in the field of statistics and operations research, especially for the one-dimension case (called Temporal Poisson Process). While estimation of homogeneous Poisson processes (HPP), where the rate is constant, is trivial, the problem of estimating nonhomogeneous Poisson processes (NHPP), where the rate is a function of time, is considered more challenging  \citep{bibi-Brown2005, Alizadeh2013, Benjamin2021}. Moreover, all of the estimation methods in the literature so far have relied on dividing observed data into bins (binning method), very often with equal length \citep{Moore2019}. While it is reasonable to use binning method for statistical estimation with infinite amount of data as it helps to reduce the estimation difficulty by transforming a NHPP into a HPP (see further explanations in section \ref{ProblemDescription}), it could create some serious problems when working with finite amount of data (which is usually the case for real-world applications). Firstly, when the amount of data is finite, dividing into bins might not guarantee that the data in each bin still follows assumed distributional assumptions, in this case - Poisson property. For instance, in all the works on estimation of NHPPs in the literature that we have studied, none of them even tested if the observed data used in their experiments actually came from a Poisson process (except for synthetic datasets coming from simulating some predefined hypothetical Poisson processes). Some even explicitly acknowledged this concerned research hole \citep{Alizadeh2013}. Secondly, with finite amount of data, the estimation becomes a learning problem, where the central issue is generalization, i.e how well the estimated/learned model could predict the future/unseen data. To the best of our knowledge, we have not seen in the literature any work that systematically formulate and solve the problem of estimation of NHPPs from finite amount of data as a learning generalization problem.  Thirdly, how the binning should be done so that it could help to solve the first and second aforementioned problems.

In this paper, we formulate the problem of estimating temporal NHPPs from finite amount of data as a learning generalization problem. Our main contributions are:
\begin{itemize}
	\item We mathematically prove that for estimation/learning NHPPs from finite amount of data, data binning could increase the probability of overfitting (i.e the learned model is good on observed/training data but bad on future/ unseen data).
	\item We propose two adaptive and data-driven binning methods that help to regularize the learning of NHPPs as well as automatically obtains the number of bins so as to remove the ad-hoc tuning of this parameter for the estimation/learning process.
\end{itemize}
The remaining part of this paper is organized as follows. The related work is briefly given in Section \ref{LitteratureReview}. Section \ref{BasicConcept} provides some backgrounds on mathematical concepts and results used later in the paper. The problem description is shown in Section \ref{ProblemDescription}. Our NHPPs learning framework and the two novel binning methods for learning regularization of NHPPs are detailed in Section \ref{algorithm}. We present and discuss our experimental results in Section \ref{ex}. The paper is concluded in Section \ref{conclusion}.
\section{Literature Review}\label{LitteratureReview}
There are a number of approaches to estimate NHPPs via estimating their rate functions in nonparemtric and parametric manner. For nonparametric approaches, \cite{Leemis2004, Arkin2000, Benjamin2021} proposed simple  estimators for estimating the cumulative rate function of NHPPs on the time interval from historical data. Their methods do not require the assumption of any parameters or weights for the rate function. The authors in \cite{Leemis2004} considered interval estimators of the cumulative rate function. The expected number of events for each interval is a piecewise-constant rate function. In \cite{Arkin2000}, the authors extended the work of \cite{Leemis1991} to solve one or more realizations (data observations) that include overlap ones on an interval. Recently, the authors in \cite{Michael2022} exploited analogies between point process and time series models. An auto-regressive moving-average model for NHPP process is proposed to disentangle cluster mechanisms in continuously observed count data of an NHPP. Their algorithm overcomes the limitation of traditional estimation by allowing for a variety of parametric specifications, which need not be confined to the Markovian case, nor to fixed intensity. An obvious deficiency of these methods is that their models require every arrival time of events in every practical realization be stored in memory to allow the generation of simulated realizations. The works in \cite{Benjamin2021, Jahani2021} also have similar weakness. Other nonparametric estimators of the rate function were introduced in \cite{bibi-Kuhl2, Komaki2021}. In \cite{bibi-Kuhl2}, the authors proposed a nonparametric estimator using wavelets and adapted the nonnegative wavelet estimation of general density functions proposed in \cite{WalterShen} to estimate the rate function. This approach's main advantage is that it can be applied to most nonstationary processes that may or may not show a long-term trend or periodic emergence. In this approach, the rate function is expressed as a nonnegative linear combination of nonnegative wavelets. However, a drawback is that they require more constraints on the rate function than needed to achieve nonnegativity. As a result, there are several limitations of their procedure in constrained convex nonlinear programming problems \citep{Komaki2021}. 

Most of parametric estimators proposed in the literature have a fixed number of parameters, so they do not suffer the storage problems when the number of observed realizations (data)  increases. Moreover, when the true rate function is known to be continuous, piecewise continuous functions could be formed for fitting the rate functions \footnote{That naturally leads to the binning of the observed data.}. The piecewise function could be linear or nonlinear such as cubic splines \citep{Alizadeh2008, Alizadeh2013}. In these works, the author formulated the problem of estimation of a NHPP as a  semidefinite programming problem (SDP), where the estimated rate functions are required to be nonnegative. However, they also doubted if the approach could be practical for large and real world problems. Their approach do not scale if either the degree or the number of variables increases, essentially due to intractable numerical calculations. In the standard SDP formulations, the condition numbers of the feasible dual solutions are exponential in the degree \citep{DPapp2017}. Theoretically, every continuous rate function can be approximated by piecewise continuous functions. However, when the number of observed data (realizations) is limited, an empirical estimator might not be asymptotically close to the true rate function if it does not lie within the assumed parametric class \citep{Komaki2021}.

To our knowledge, almost all the works in the literature have been on the estimation of NHPPs, where the estimators will converge to the true rate function when the amount of observed data goes to infinity. None of them has focused on the case of learning a NHPP from finite data for generalization over future and unseen data.

\section{Background} \label{BasicConcept}
\subsection{Poisson Distribution}
The Poisson distribution, introduced by Sim\'{e}on Denis Poisson \citep{Poisson}, characterizes random variables $\mathcal{X}$ used for counting the number of events occurring in an interval. Let the average number of events in an interval be  $\lambda$, $\mathcal{X}$ has Poisson distribution if its possible values are positive integers and the probability of observing $\xi$ events ($\xi \in \mathbb{N^+}$) in the interval is as follows:
\begin{align*}
	P(\xi \textit{ events in the interval}) = \frac{e^{-\lambda} \lambda^\xi}{\xi !}
\end{align*}
For the Poisson distribution, the first two moments, expectation and variance, equal to $\lambda$ \citep{Suhov}:
\begin{align*}
	\mathbb{E}[\mathcal{X}] = \Sigma_{\xi \geq 0} \frac{e^{-\lambda} \lambda^\xi}{\xi !} \xi = \lambda
\end{align*}
\begin{align*}
	\textnormal{Var}[\mathcal{X}] = \Sigma_{\xi \geq 0} \frac{e^{-\lambda} \lambda^\xi}{\xi !} \xi^2 - \lambda^2= \lambda
\end{align*}
We shall denote $\mathcal{X}$ as $Poisson(\lambda)$.
\subsection{Temporal Poisson Process}
One of the most fundamental and widely-applied stochastic point processes is the temporal Poisson process used for characterizing discrete event sets in continuous time. It is usually utilized for modelling the situations where we want to count the occurrences of certain events that appear entirely at random, but with a certain rate (without a particular structure) \citep{Levy}. 

\subsubsection{Temporal Homogeneous Poisson Process}
Let $\{\mathcal{N}(t)$, $t \in \mathbb{R}^+\}$ be a counting process for counting the occurrences of events that occurred from time $0$ up to and including time $t$. Given $0 < t_1 < t_2 < \infty$, $\mathcal{N}(t_2) - \mathcal{N}(t_1)$ is the number of events occurred in the time interval $(t_1, t_2]$. 

\begin{definition} \citep{Ross2009}
	$\{\mathcal{N}(t)$, $t \in \mathbb{R}^+\}$ is said to have independent increments if for any selection of distinct time-points $0 < t_1 < t_2 < \dots < t_m$, the random vectors $\mathcal{N}(t_1), \mathcal{N}(t_2) - \mathcal{N}(t_1), \dots, \mathcal{N}(t_m) - \mathcal{N}(t_{m-1})$ are all independent.
\end{definition}
\begin{definition} \citep{Ross2009}
	$\{\mathcal{N}(t)$, $t \in \mathbb{R}^+\}$ is said to have stationary increments if for all $0 \leq t_1 \leq t_2 < \infty$, $\mathcal{N}(t_2) - \mathcal{N}(t_1)$ has a distribution that only depends on the length of the time interval $(t_2 - t_1)$.
\end{definition}
There are several equivalent definitions for a Poisson process. We present here the simplest one \citep{Ross2009}.

\begin{definition}
	$\{\mathcal{N}(t)$, $t \in \mathbb{R}^+\}$ is called the temporal homogeneous Poisson process with arrival rate $\lambda$ if:
	\begin{itemize}
		\item $\mathcal{N}(0) = 0$.
		\item $\mathcal{N}(t)$ has independent increments.
		\item $\mathcal{N}(t_2) - \mathcal{N}(t_1) \sim Poisson(\lambda (t_2 - t_1))$ for any two distinct times $0 \leq t_1 < t_2 < \infty$.
	\end{itemize}
\end{definition}
From the above definition, we could see that a homogeneous Poisson process has both stationary and independent increments.

\subsubsection{Temporal Nonhomogeneous Poisson Processes}
\label{TNHPP}
Let $\{\mathcal{N}(t)$, $t \in \mathbb{R}^+\}$ be a counting process, $\{\mathcal{N}(t)$, $t \in \mathbb{R}^+\}$ is a NHPP with arrival rate $\lambda(t)$ if:
\begin{itemize}
	\item $\mathcal{N}(0) = 0$.
	\item $\mathcal{N}(t)$ has independent increments.
	\item $\mathcal{N}(t_2) - \mathcal{N}(t_1) \sim Poisson(\int_{t_1}^{t_2} \lambda(u)du)$ for any two distinct times $0 \leq t_1 < t_2 < \infty$.
\end{itemize}
We notice that, contrary to the homogeneous case, a NHPP does not have stationary increments \citep{Ross2009,bibi-Cinlar}.

\subsection{Distribution Tests}
This section gives the descriptions of two main statistical tests for Poisson distribution. First, the well-known Kolmogorov-Smirnov test (KS) is introduced. Second, we describe an alternative statistical test based on the KS test after performing a logarithmic transformation of the data as proposed in \cite{bibi-Brown2005}.
\subsubsection{Standard Kolmogorov-Smirnov Test} $\newline$
Suppose that the null and alternative hypotheses of interest are:
\begin{align*}
	H_0 :& \textnormal{ The data come from a Poisson distribution.}\\
	H_1 :& \textnormal{ The data do not follow the Poisson distribution.}
\end{align*}
The KS test procedure compares the observed data with a specified theoretical distribution. This test determines the maximum distance between the empirical distribution function $F^*$ of the samples and the theoretical cumulative distribution function (CDF) of the known distribution $F$, where  $F^*$ for $m$ independently and identically distributed (i.i.d.) observations $\mathcal{X}_i$ is defined as:
\begin{align*}
	F^*(x) = \frac{1}{m} \sum_{i=1}^{m} \mathcal{I}_{(-\infty, x]}(\mathcal{X}_i)
\end{align*}
where the indicator function $\mathcal{I}_{(-\infty ,x]}(\mathcal{X}_i)$ is equal to 1 if $\mathcal{X}_i \leq x$ and 0, otherwise. The KS statistic for a given CDF $F(x)$ is calculated as follow:
\begin{align} \label{measureKS}
	D =\sup_x |F(x)- F^*(x)|
\end{align}	
The null hypothesis is rejected if $D$ is greater than the critical value $D_{m, \epsilon}$ obtained from a formula, where $\epsilon$ is the significance level. There are several variations of this formula in the literature that use somewhat different scalings for $D$ and critical regions \citep{Lilliefors, Durbin193, Sach1997}, etc. As in familiar situations, we use the formula $D_{m, \epsilon} = \sqrt{(-0.5 \ln(\epsilon/2)/m}$ in \cite{Sach1997}.

\subsubsection{Log Test} $\newline$
For each subinterval $[l, uk]$, we define:
\begin{align*}
	X_{i} = -(m + 1 - i)\log(\frac{u - l - t_i}{u - l - t_{i-1}}), \textnormal{ } 1 \leq i \leq m
\end{align*}
where $m$ denotes the total number of events and $t_i$ is the arrival time of the $i^{th}$ event.\\
Under $H_0$, within each given interval,  $\{X_i\}$ will be independent standard exponential variables. Therefore, any standard test for the exponential distribution can be applied to test the null hypothesis. For convenience, we use the KS test. Equation (\ref{measureKS}) can then be applied using the exponential CDF $F(x) = 1 - e^{-x}$. The KS test is widely used for computational convenience, even though it may not have the most significant possible power against the alternatives of most interest. Compared to the KS test, \cite{bibi-Kim2014} showed the Log Test has higher power, i.e., that it has much greater power against these alternatives. Therefore, this test is used for our proposed method in this paper.
\subsection{Statistical Learning Theory}
In this section, we summarize some main concepts and results of the statistical learning theory given in \cite{Vapnik1998} that will serve as the foundations for the problem of learning NHPPs.
\subsubsection{Risk Functional Minimization}
Given a set of observations $z_1, z_2, \dots, z_m$ coming from an unknown probability distribution $F(z)$, many statistical learning problems based on these observed data could be formulated as minimizing the risk functional
\begin{align}\label{riskfunctional}
	R(\alpha) = \int Q(z, \alpha) dF(z)
\end{align}
where $Q(z, \alpha)$ is a specific loss function, $\alpha$ is a parameter in the set of parameters $\Lambda$.
\subsubsection{The Problem of Regression Estimation/ Learning}
\label{Regression}
One of the common learning problems is regression estimation, which could be defined as a risk functional minimization problem. Given $G$ as the generator of vector $x$, where $x \in \mathcal{X} \subset \mathbf{R}^d$ is i.i.d. observations according to some unknown (but fixed) probability distribution function $P(x)$.
The learning machine observes $m$ pairs	$(x_1, y_1), (x_2, y_2), \dots, (x_m, y_m)$ where, $(x_1, \dots, x_m)$ and $(y_1, \dots, y_m)$ are the input vectors and the supervisor's responses, respectively. Since the input vectors drawn randomly and independently from $P(x)$, the supervisor's responses are obtained at random from $P(y|x)$. In this case, there exists a joint (but unknown) distribution function $P(x, y)$.
The primary purpose of the learning machine is to construct an appropriate approximation, which will be used for predicting the supervisor's response $y_i$ on any input vector $x_i$ generated by the generator $G$. 

The approximating function is chosen from a given set of functions $\overline{F} = \{ f(x, \alpha) \emph{ } | \emph{ } f \in L_2(P), \alpha \in \Lambda\}$.  To choose the best function, we measure the loss or difference between the supervisor's response to a given input and the learning machine's response. Let the function of conditional  expectation is given by:
\begin{align*}
	r(x) = \int y dP(y|x)
\end{align*}
This function is called the regression function. We assume that:
\begin{align*}
	\int y^2 dP(x, y) < \infty, \\
	\int r^2(x) dP(x, y) < \infty
\end{align*}
Since $f(x, \alpha) \in L_2(P)$, the minimum of the risk function
\begin{align}\label{risk}
	R(\alpha) &= \int (y - f(x, \alpha))^2 dP(x,y)
\end{align}
is attained at the regression function $r(x)$ (or the closest to $r(x)$) \citep{Vapnik1998}.

\subsubsection{Empirical Risk Minimization (ERM) Principle}
A learning machine is capable of implementing a set of functions $\overline{F}$. The problem of learning is to choose a function $f(x, \alpha)$ from $\overline{F}$ that minimizes the risk function (\ref{risk}). In order to define the regression problem, we introduce an $(d+1)$-dimensional variable $z = (x,y)$ and $z_1, z_2, \dots, z_m$ come from an unknown probability distribution $P(x, y)$
\begin{align*}
	z_i = (x_i, y_i), \emph{ } x_i \in \mathbf{R}^d, y_i \in \mathbf{R}, i = 1, \dots, m
\end{align*}
We cannot compute $R(\alpha)$ directly but only evaluate the empirical risk
\begin{align}\label{riskemp}
	R_{emp} (\alpha) = \frac{1}{m} \sum_{i=1}^m Q(z_i, \alpha), \emph{   } \alpha \in \Lambda
\end{align}
in the situation where the probability distribution $P(x, y)$ is unknown and the only available information is contained in the set of the observed data. The set of loss functions $Q(z, \alpha), \alpha \in \Lambda$ is of the form:
\begin{align*}
	Q(z, \alpha) = (y - f(x, \alpha))^2 
\end{align*}

The Glivenko-Cantelli theorem \citep{Vapnik1998} provides the theoretical foundation for the ERM principle in the asymptotic case (i.e. when the number of observations goes to infinity). It explains why ERM should give good results for learning problems with a large sample size.
However, for a limited amount of empirical data, in general, a small value of empirical risk does not guarantee a small value of the actual risk. That leads to the problem of generalization of ERM based learning machines with small and limited training data \citep{Vapnik1998}.


\subsubsection{Generalization of learning machines}
\label{generalizationoflearning}
As mentioned in the previous subsection, a small empirical risk (training error) value does not guarantee a small actual risk (true error) for a learning machine trained on limited data. It means that a learning machine could achieve small errors on a (limited size) training dataset but perform poorly on new (future and unseen) data. In that case, the learning machine is called overfitting. To avoid overfitting, i.e. to generalize the good performance on the training data to future and unseen data is the most desirable property that a learning machine should have. This property is called learning generalization. 

In practice, the generalization capacity of a learning machine is usually measured by comparing the errors of the learning machine on two independent datasets - the training dataset (used in the training phase) and the test dataset (used in the testing phase). In theory, it could be controlled by using the constructive risk bound given in the statistical learning theory. In particular,
Vapnik \citep{Vapnik2000} proved that with probability at least $1 - \eta$, $(0<\eta<1)$ simultaneously for all functions from the set of totally bounded functions $0 \leq Q(z, \alpha) \leq B, \alpha \in \Lambda$, with the finite VC dimension $h$, the inequality
\begin{align}\label{boundofrisk}
	R(\alpha) \leq R_{emp}(\alpha) + 2B \xi \left(  1 + \sqrt{1+\frac{R_{emp}(\alpha)}{B\xi}} \right)
\end{align}
holds true, where
\begin{align}\label{xi}
	\xi = \frac{h\lbrack\ln{(2m/h)} + 1\rbrack - \ln{(\eta/4)}}{m}
\end{align}
It can be seen from inequality (\ref{boundofrisk}) that to ensure that the true risk is small (i.e. good learning generalization), both empirical risk (training error) of the learning machine and the VC dimension of the function (hypothesis) space must be small. In other words, for a guaranteed generalization performance of a learning machine trained on a limited set of data, the training error should be small and the learning function should be chosen from a set of functions with as small VC dimension as possible. However, these two requirements are contradictory since a hypothesis space as a low VC dimension set of functions could be too narrow to contain the functions that could achieve small empirical risk values (in this case, the learning machine is called underfitting). Therefore, a successful statistical learning scheme should balance between the choice of the hypothesis space and the chance to get a function in that hypothesis space with small empirical risk (training error). One of such schemes is the structural risk minimization principle.

\subsubsection{Structural risk minimization principle}
The structural risk minimization (SRM) principle,  proposed in \cite{Vapnik1998}, is an inductive principle for model selection used for learning from finite training datasets. It describes a general way of capacity control that provides a trade-off between the chosen model's complexity and the empirical error.\\
Let us define a structure of nested subsets:
\begin{align*}
	\mathcal{H}_1 \subset \mathcal{H}_2 \subset \dots \subset \mathcal{H}_r \subset \dots
\end{align*}
where, $\mathcal{H}_k = \{ Q(z, \alpha_k) \emph{ }| \emph{ } \alpha_k \in \Lambda_k\}$. A structure is called admissible if it ensures that the capacity $h_k$ of $\mathcal{H}_k$ is less than $h_{k+1}$ of  $\mathcal{H}_{k+1}$ (see \citealp{Vapnik1998} for more details). The SRM principle states that the subset and the function that minimizes the empirical risk within such subset (e.g., $\mathcal{H}^*$ in Fig. \ref{figure-SRMP}), so as to minimize risk bound, yields the best overall generalization performance. Fig. \ref{figure-SRMP} gives the graphical explanation of the SRM principle.
\begin{figure}[!t]
	\centering
	\includegraphics[width=2in]{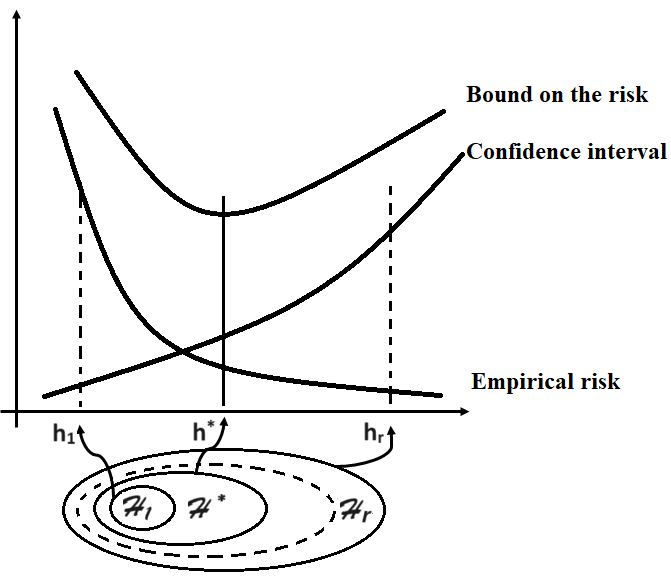}
	\caption{Structural risk minimization principle of Vapnik \citep{Vapnik1998}}
	\label{figure-SRMP}
\end{figure}


\subsubsection{Regularization methods}
The regularization method was proposed by A. N. Tikhonov \citep{Tikhonov}. This method discourages learning a more complex or flexible model to avoid the risk of overfitting and increase generalization. Regularization can be as simple as shrinking or penalizing large coefficients, often called weight decay. The idea of regularization is to find a solution for the empirical risk functional (\ref{riskemp}) that is modified by adding the shrinkage quantity.\\
$\textbf{Tikhonov regularization}$ minimizes the functional:
\begin{align*}
	R_{reg} (\alpha) = R_{emp} (\alpha) + \gamma * \textit{Reg}(\alpha),
\end{align*}
where $\gamma$ is called \emph{regularization parameter} which controls the importance of the regularization functional $\textit{Reg}(\alpha)$. The regularization functional $\textit{Reg}(\alpha)$ is chosen to impose a penalty on the complexity of the model. For the case of Hilbert space, we typically choose $\mathcal{L}_2$-regularization that is differentiable:
\begin{align*}
	\textit{Reg}(\alpha) = ||\alpha||^2_2
\end{align*}
The coefficients are estimated by minimizing the functional $R_{reg}(\alpha)$. Therefore, regularization is the process that regularizes or shrinks the coefficients towards zero. This straightforward solution is one of the main reasons for the use of Tikhonov regularization in the Hilbert space setting.\\
$\textbf{Ivanov regularization}$ is yet another method that solves:
\begin{align*}
	&\min_{\alpha} R_{emp}(\alpha),\\
	&s.t. \emph{ } \alpha \leq \tau
\end{align*}
The Ivanov regularization allows us to handle the two main forces directly: on the one hand, the depreciation of the functional $R_{emp}$ and, on the other hand, the control of the hypothesis space, where the hypothesis minimizing the risk is chosen from.

\section{Problem Description} \label{ProblemDescription}
	For learning a NHPP, it is of primary importance to accurately estimate the underlying rate function since the entire NHPP could be generated using its rate function. For instance, the author in \cite{Streit2010} specified that realizations of $\mathcal{N}(t)$ on a specified finite subset $S$ could be generated by sampling independently on each of the Poisson variates in $\mathbf{R}^+$. He proposed a two-step procedure to generate a NHPP using its rate function $\lambda(t)$ on a bounded subset $S$.
	Therefore, for learning a NHPP, our main focus is to learn its rate function from limited observed data. 
	
	
	\subsection{Problem of Learning Temporal NHPPs}
	Given a temporal NHPP $\mathcal{N}(t)$ (as defined in  \ref{TNHPP}) with a fixed but unknown rate function $\lambda(t)$, consider time series (the full history of the NHPP) of a NHPP that are observed $\mathcal{F}_{[a,b]} = \{t_1, t_2, \dots, t_m\}$ ($a\leq t_i \leq b, \forall i = \overline{1,m}$). The observed dataset $D$ comprises of pairs $(t_i, y_i)$, where $y_i$ is the number of events occurred in the time interval $(t_{i-1}, t_i]$, i.e., $\mathcal{N}(t_i) - \mathcal{N}(t_{i-1})$ ($\forall i = \overline{2,m}$). As aforementioned, the problem of learning $\mathcal{N}(t)$ is reduced to learning the rate function $\lambda(t)$ from the finite dataset $D$, which could be cast to the problem of function regression as in \ref{Regression}. Under the setting of the ERM principle, the problem is to choose a function $\lambda(t, \alpha)$ from a function set $\overline{F}$ that minimizes the empirical risk functional (\ref{riskemp}). 
	\subsection{Generalization of Learning}
	\label{LearningG}
	To approximate the rate function, a popular approach in the literature is function estimation with infinite data, where a separation between training and testing phase is not necessary. In an estimation process, a natural first step is data binning, a frequently-used technique in recent literature \citep{bibi-Brown2005, Streit2010, Alizadeh2013, Ferrucci2013, Hsieh2017}. The common estimation process using data binning is given in Fig. \ref{figure-commonProcess}.
	\begin{figure}
		\centering
		\includegraphics[width=3in]{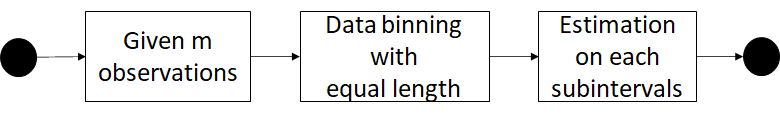}
		\caption{Common estimation process}
		\label{figure-commonProcess}
	\end{figure}
	When the length of each bin is equal (which is very much often in the literature), this approach has a single parameter, the number of bins, and so it is possible to obtain a more refined or coarser description of the data by controlling the number of bins. Therefore, it could be regarded as a deterministic transformation from a NHPP to a HPP (for example, see \citealp{Streit2010}, Section 2.10) when the number of bins gets increased. In the learning process, this approach usually restricts its attention to special parametric families of rate functions to improve the learning quality. 
	Learning on small subintervals can make an accurate approximation for all observed data. The powerful reason motivating the use of this model selection is the reduction of the complexity of the chosen functions and the reduction of training errors. Therefore, data binning is necessary and pervasive in both estimation and learning process to reduce the empirical risk value. This could be shown by the following proposition.\\
	Given $m$ observations, the functionals in (\ref{risk}), (\ref{riskemp}) and (\ref{xi}) can be rewritten as $R(\alpha, m)$, $R_{emp}(\alpha, m)$ and $\xi(m)$ related to the number of observations, $m$. Suppose the range of considered time $[0, T]$ is divided into $n$ consecutive nonoverlapping subintervals by knots $b_0 = 0 < b_1 < \dots < b_n = T$. 
	Let $t_i$, $0\leq t_i \leq T$, be arrival-time point of event $x_i$, $S_k = \{ z_i = (x_i, y_i) \emph{ } | \emph{ } b_{k-1} < t_i \leq b_{k}\wedge i = \overline{1,m}\}$, $\forall k = \overline{1, n}$. We denote the empirical risk functional on the $k^{th}$ subinterval as $R^k_{emp} (\alpha_k, m_k)$, where, $m_k = |S_k|$. The empirical risk functional after binning is computed as:
	\begin{align}
		\hat{R}_{emp}(\alpha_1, \dots, \alpha_n, m) = \frac{1}{m}\sum_{k=1}^n m_k R^k_{emp}(\alpha_k, m_k) \label{riskSum}
	\end{align}
	
	\begin{proposition}
		For any $n \geq 2$, the inequality
		\begin{align*}
			\min_{\alpha_1, \dots, \alpha_n \in \Lambda} \hat{R}_{emp} (\alpha_1, \dots, \alpha_n, m) \leq \min_{\alpha \in \Lambda} R_{emp}(\alpha, m)
		\end{align*}
		holds true.
	\end{proposition}
	\begin{proof}
		For each subinterval $k$, let us denote $\hat{\alpha}_k = \arg \min_{\alpha_k \in \Lambda} R^k_{emp} (\alpha_k, m_k)$ and $\hat{\alpha} = \arg \min_{\alpha \in \Lambda} R_{emp}(\alpha, m)$. The inequality in the proposition could be inferred from the following inequality:
		\begin{align*}
			R^k_{emp} (\hat{\alpha}_k, m_k) \leq R_{emp}(\hat{\alpha}, m_k)
		\end{align*}
		for each subinterval $k = \overline{1, n}$.
	\end{proof}
	Proposition 1 shows that dividing the training data into smaller bins likely leads to the decrease of empirical risk (better approximation of training data).
	
	In prior works, the common approach is to utilize the data binning with equal length \citep{bibi-Brown2005, Alizadeh2013, Ferrucci2013} and this is independent of the estimation step. The number of subintervals is typically dependent on the modeler's experience and not data driven. Having a small number of subintervals will lead to a large empirical risk value, while having a larger number of subintervals may improve the empirical risk but may be risky. Here, one of the risks is that training data is sparse and nonhomogeneous in time, therefore, there are almost always empty subintervals. The other risk is generalization and the following questions could be asked:
	\begin{itemize}
		\item \emph{Does the binning on finite data increase the probability of overfitting?}
		\item \emph{How to divide the data into bins to avoid overfitting?}
	\end{itemize}
	To answer the first question, according to \cite{Vapnik2000}, we consider the relationship between a binning method on finite data and the overfitting problem. In this context, we formulate a proposition to show that the generalization of a learning machine depends not only on the empirical risk value and the VC dimension of $F$ (see Section \ref{generalizationoflearning}) but also on the binning method. In detail, if we use the smaller number of observations (i.e., the smaller size of subintervals) for approximation, 
	the error on the training set can be reduced but the error on the testing set may  increase.\\
	We rewrite the right-hand side of  (\ref{boundofrisk}) in the new form
	\begin{align*}
		\overline{R}(\alpha, m) = R_{emp}(\alpha, m) + 2B \xi(m) \left( 1 + \sqrt{1+\frac{R_{emp}(\alpha, m)}{B\xi(m)}} \right)
	\end{align*}
	Then functional (\ref{boundofrisk}) can be written in as follow:
	\begin{align*}
		R(\alpha, m) \leq \overline{R}(\alpha, m) 
	\end{align*}
	\begin{proposition}
		There exists a number $m^*$ such that for all $m < m^*$, we have
		\begin{align*}
			\overline{R}(\alpha, m) \geq \overline{R}(\alpha, m^*)
		\end{align*}
	\end{proposition}
	\begin{proof}
		Clearly, from the definition of the empirical risk functional, we have $R_{emp}(\alpha, m) \xrightarrow[m \rightarrow \infty]{a.s} R(\alpha, m)$. Therefore, the functional $R_{emp} (\alpha, m)$ is bounded from below. We denote
		\begin{align*}
			m^* = \arg \min_m R_{emp}(\alpha, m),\\
		\end{align*}
		where $m >> h$ (see \citealp{Vapnik2000}).
		Moreover, we have:
		\begin{align*}
			\xi(m) = \frac{h\lbrack\ln{(2m/h)} + 1\rbrack - \ln{(\eta/4)}}{m}
		\end{align*}
		where $0 < \eta < 1$. Thus, we obtain $\ln{\eta/4} < 0$ and the functional $(\ln{2x})/x$ is monotonically decreasing on $[e^2/2, +\infty)$. Therefore, we obtain
		\begin{align*}
			\xi(m) & \left(1 + \sqrt{1 + \frac{R_{emp}(\alpha, m)}{B\xi(m)}} \right) \geq\\
			& \xi(m^*) \left(1 + \sqrt{1+\frac{R_{emp}(\alpha, m^*)}{B\xi(m^*)}}\right)
		\end{align*}
		This proposition is proved.
	\end{proof}
	We see that the number of observations on each subinterval depends on the number of subintervals. Basically, the larger number of subintervals leads to a smaller number of observations on subintervals. Therefore, this proposition specifies that binning the subintervals into smaller ones carries a risk of overfitting.\\
	To illustrate the quality of the learning process depending on the number of subintervals for a tested problem and data described in section \ref{ex}, Fig. \ref{figure-overfitting-SF} shows how the training (solid line) and test (dashed line) errors evolve when the number of subintervals gets increased and the binning method is equal length binning. We could see that the test error increases while the training error steadily decreases starting from 15 subintervals (bins). A typical error graph for the overfitting phenomenon is observed.\\
	\begin{figure}
		\centering
		\includegraphics[width = \linewidth]{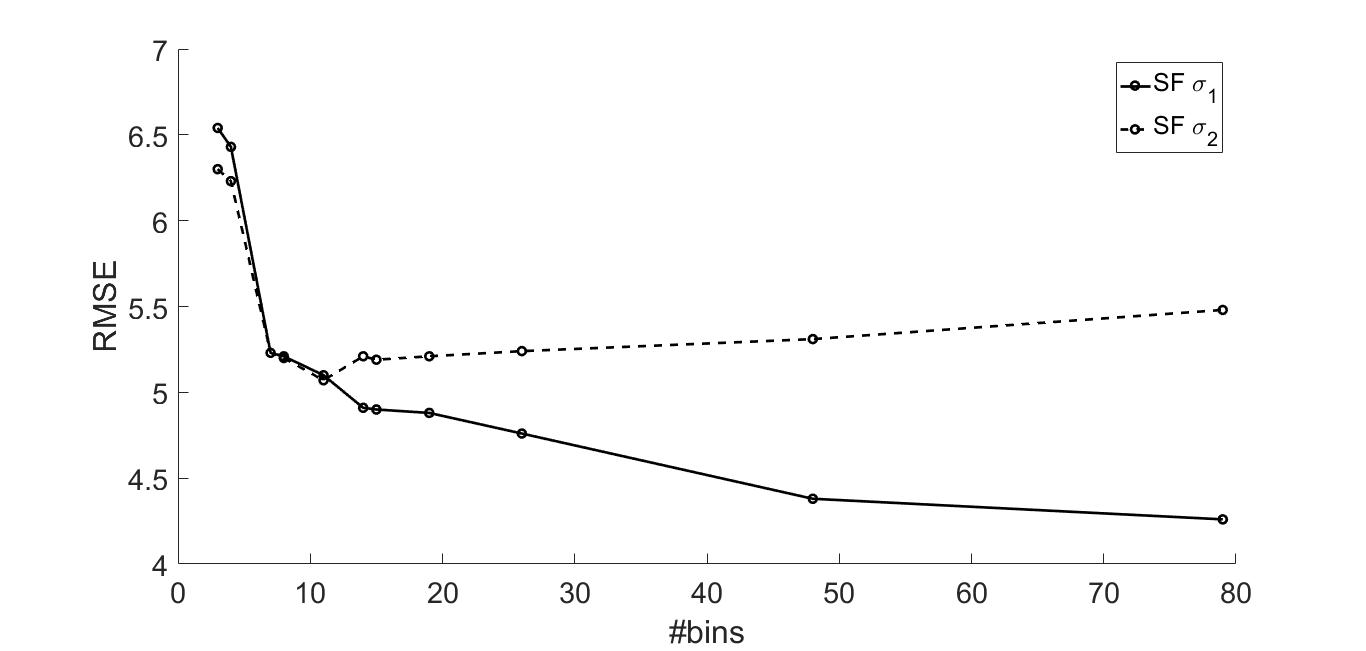}
		\caption{Overfitting in piecewise-polynomial regression on the San Francisco taxi request data (see Section \ref{ex})}
		\label{figure-overfitting-SF}
	\end{figure}
	To answer the second question, we propose, in the next section, two regularization methods for binning the training data, one is based on Ivanov regularization and the other is based on Tikhonov regularization.
	\section{Learning Regularization of NHPPs} \label{algorithm}
	In this section, we propose a general framework for learning NHPPs. Two novel binning methods inspired by Ivanov and Tikhonov regularizations are given. Finally, the general regularized learning algorithm of NHPPs is detailed.
	\subsection{Learning Framework}
	In this part, we propose a process of learning the rate function of a NHPP. Our process is decomposed into two steps: data binning and regression. In the data binning step, samples are placed in discrete bins by a binning method. After that, a regression learning technique is applied. Our learning process is different from the literature in that the binning step is not conducted one time but repeatedly with the regression step. This makes our binning step more adaptive to the training data. The whole learning process is illustrated in Fig. \ref{figure-proposedProcess}.\\
	\begin{figure}
		\centering
		\includegraphics[width = \linewidth]{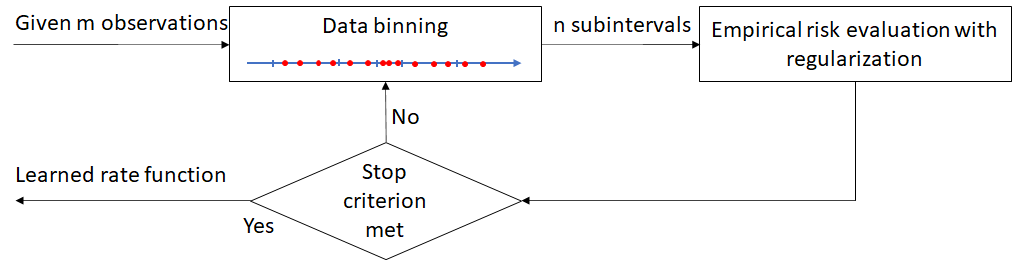}
		\caption{The proposed learning framework}
		\label{figure-proposedProcess}
	\end{figure}
	
	\subsection{Data Binning Methods}
	As previously discussed, binning is important for learning/ estimation of NHPPs. While having too few bins could result in bad estimations of NHPPs (underfitting), using too many bins could risk overfitting. In general, a learning/estimation of NHPPs is depend on not only the number of bins but also the way of dividing the data into bins. However, to the best of our knowledge, all of the previous works on estimation/learning of NHPPs have sidelined this problem and taken ad-hoc, non-data-driven approaches in choosing the way to divide data into bins (e.g., divide data into equal-length intervals with a predefined bin number). In this subsection, we propose two adaptive and data driven binning methods that try to take into account the learning regularization.  
	
	\subsubsection{Ivanov regularization based data binning method}
	Based on the idea of Ivanov regularization, we propose a general data division algorithm using distribution tests for breaking up the considered time interval into subintervals. 
	The distribution test plays the role of the constraint in Ivanov regularization, which is used to restrict the bins to the set of subintervals that respect the Poisson property of training data. The details of the algorithm are shown in Algorithm \ref{sim:dividedGeneralAlgo}.  In the algorithm, $dP_{best}$ indicates the best solution found during the search. It divides an interval into subintervals (lines 2-3) by a uniformly random point (line 1). Then, the Poisson property on each subinterval is tested using the Log Test \citep{bibi-Brown2005} (line 4). Lines 5-13 update the best training error and recursively work on the newly created subintervals. Lines 14-18 continue the division if the stopping criterion (for instance, the recursive depth or the total number of bins) is not met.\\

	\begin{algorithm}
		\caption{Recursive data binning method based on the Ivanov regularization}
		\label{sim:dividedGeneralAlgo}
		\begin{algorithmic}[1]
			\PROCEDURE{\textsc{dividedInterval}}{$(l, u, O, dP)$}
			\Require
				\Statex l: lower bound of interval
				\Statex u: upper bound of interval
				\Statex O = $\{ t_{1}, ..., t_m \}$: arrival-time points
				\Statex dP: knots
			\Ensure minR, $dP_{best}$: the best-found risk value and knot list
			\State $p \leftarrow$ random number in U $\simeq$ $\mathrm{Uniform}(l, u)$
			\State $O_1 \leftarrow \{t_{i} \emph{ }| \emph{ } t_i \in O \wedge l \leq t_{i} < p \}$
			\State $O_2 \leftarrow \{t_{i} \emph{ }| \emph{ } t_i \in O \wedge p \leq t_{i} < u \}$
			\State Using a distribution test for testing Poisson property on $O_1$ and $O_2$
			\If{Both $O_1$ and $O_2$ are passed the distribution test}
				\State $dP \leftarrow dP \cup {p}$
				\State Let $er$ be the value of risk functional (\ref{riskSum})
				\If{$er \leq minR$}
					\State $minR \leftarrow er$
					\State $dP_{best} \leftarrow dP$
				\EndIf
				\State \Call{dividedInterval}{$l, p, O_1, dP$}
				\State \Call{dividedInterval}{$p, u, O_2, dP$}
			\Else
				\If{stop-criterion not met}
					\State \Call{dividedInterval}{$l, u, O, dP$}
				\EndIf
			\EndIf
			\ENDPROCEDURE
		\end{algorithmic}
	\end{algorithm}
	
	\subsubsection{Data binning method based on the idea of Tikhonov regularization}
	Based on the idea of Tikhonov regularization, we propose a binning method where an additional regularized term is added to the empirical risk function. This term is used to control the number of subintervals based on their lengths. The empirical risk functional (\ref{riskSum}) is reformulated as:
	\begin{align} \label{riskTikhonov}
		\hat{R}_{emp} (\alpha_1, \dots, \alpha_n, m) = & \frac{1}{m}\Sigma_{k=1}^{n} m_k R^k_{emp} (\alpha_k, m_k) \nonumber\\
		& + \gamma \Sigma_{k=1}^{n-1} \delta_{k}
	\end{align}
	where $\delta_{k} = m_k R^{k}_{emp} (\alpha_k, m_k) / (b_k - b_{k-1})$ is the risk value scaled by the length of subinterval $k^{th}$ found at the previous division step.  When we increase the number of subintervals, in general, the length of subintervals and the risk value are decreased. However, the decreasing rate of the risk value is usually slower than that of the subinterval length. Thus, when the number of subintervals is large enough, the penalty term tends to increase. This form of regularization term represents the momentum of the empirical risk function related to the time interval after each division iteration.
	The proposed idea is depicted in Algorithm \ref{sim:dividedGeneralAlgoTikhonov}. Line 1 generates a divided point randomly. Line 3 computes the new empirical risk value. Lines 4-7 save the best solution found. Lines 8-11 continue the division if the stopping criterion (for instance, the recursive depth or the total number of bins) is not met.\\
	
	\begin{algorithm}
		\caption{Recursive data binning method based on the Tikhonov regularization}
		\label{sim:dividedGeneralAlgoTikhonov}
		\begin{algorithmic}[1]
			\PROCEDURE{\textsc{dividedInterval}}{$(l, u, dP)$}
			\Require
				\Statex l: lower bound of interval\;
				\Statex u: upper bound of interval\;
				\Statex dP: knots\;
			\Ensure minR, $dP_{best}$: the best-found risk value and knot list
			\State $p \leftarrow$ random number in U $\simeq$ $\mathrm{Uniform}(l, u)$\;
			\State $dP \leftarrow dP \cup {p}$\;
			\State Let $er$ be the value of risk functional (\ref{riskTikhonov})\;
			\If{$er \leq minR$}
				\State $minR \leftarrow er$\;
				\State $dP_{best} \leftarrow dP$
			\EndIf
			\If{stop-criterion not met}
				\State \Call{dividedInterval}{$l, p, dP$}
				\State \Call{dividedInterval}{$p, u, dP)$}
			\EndIf
			\ENDPROCEDURE
		\end{algorithmic}
	\end{algorithm}
		
	
	\subsection{A Proposed Algorithm for Estimating the Rate Function of NHPPs}
	Algorithm \ref{sim:binning}  is an iterative algorithm for estimating the rate function of a NHPP. It repeatedly calls Algorithm \ref{sim:dividedGeneralAlgo} (or Algorithm \ref{sim:dividedGeneralAlgoTikhonov}) until the stopping criterion is met. In the algorithm, $minR$ and $dP_{best}$ indicate the best risk value and binning solution found during the search.
	\begin{algorithm}
		\caption{Algorithm for learning the NHPP rate function}
		\label{sim:binning}
		\begin{algorithmic}[1]
			\PROCEDURE{\textsc{dividedInterval}}{}
			\Require
				\Statex [0, T]: considered time interval
				\Statex O = $\{ t_1, ..., t_m \}$: arrival-time points
			\Ensure minR, $dP_{best}$: the best-found risk value and knot list
			\State Initialize $dP_{best} \leftarrow dP \leftarrow \{\emptyset\}$\;
			\State $minR \leftarrow \infty$\;
			\While{stop-criterion not met}
				\State $dP \leftarrow \Call{dividedInterval}{0, T, O, dP}$
				\State Let $er$ be the value of risk functional (\ref{riskSum})
				\If{$er < minR$}
					\State $minR \leftarrow er$\;
					\State $dP_{best} \leftarrow dP$\;
				\EndIf
				\State $dP \leftarrow \{ \emptyset \}$\;
			\EndWhile
			\ENDPROCEDURE
		\end{algorithmic}
	\end{algorithm}
		
	\section{Experiments and Results} \label{ex}
	\subsection{Selection of Hypothesis Space}
	To test the effect of binning on learning generalization of NHPPs, we select hypothesis spaces with fixed VC-dimensions in each experiment. This practice is common in the literature of NHPP estimation \citep{Massey1996, Leemis2004, bibi-Brown2005, Alizadeh2008, Alizadeh2013}. The two most popular function spaces for NHPP estimation in the literature are selected, namely, the spaces of linear and polynomial functions \citep{Massey1996, Alizadeh2013}.
	
	\subsection{Tested Problems and Data}
	\label{expproblemdata}
	In the experiments, we have chosen some well-known NHPPs learning/estimation problems from the literature that model different kinds of events such as daily taxi call arrivals (e.g., \citealp{SonNV, Yoon2019, Liu2019}) and telephone call arrivals at customer service call centers (e.g., \citealp{bibi-Brown2005, Shen2008, Xu2017}). For these problems, the true rate functions are unknown and must be learned/ approximated from data.
	Besides, we also consider the problem of learning a NHPP from a synthetic dataset, which was randomly generated  with a known piecewise-linear rate function. These problems and their data are described as follows.
	\begin{itemize}
		\item Problem 1 (AF): Learning the arrival rate of randomly generated requests coming from a NHPP. The synthetic dataset is generated by a standard algorithm described in \cite{Ross2013} with the piecewise-linear rate function $\lambda(t)$ as follows:
		\begin{align*}
			\lambda(t) = 
			\begin{cases}
				\cfrac{13}{36} \times t + 7, & \quad \text{if } 0 \leq t < 10800,\\
				\cfrac{5}{36} \times t + 15, & \quad \text{if } 10800 \leq t < 21600,\\
				\cfrac{25}{36} \times t - 25, & \quad \text{if } 21600 \leq t < 32400,\\
				-\cfrac{1}{2} \times t + 104, & \quad \text{if } 32400 \leq t < 43200,\\
				-\cfrac{1}{18} \times t + 40, & \quad \text{if } 43200 \leq t < 54000,\\
				\cfrac{1}{3} \times t - 30, & \quad \text{if } 54000 \leq t < 64800,\\
				-\cfrac{4}{9} \times t + 138, & \quad \text{if } 64800 \leq t < 75600,\\
				-\cfrac{5}{9} \times t + 166, & \quad \text{if } 75600 \leq t < 86400,\\
				0, \quad \text{otherwise.}
			\end{cases}
		\end{align*}
		Subintervals of lengths are equal, using known periodic arrival rate functions. For each day, the number of requests is randomly chosen in [7000, 8000].
		\begin{itemize}
			\item Training dataset: 365 days.
			\item Testing dataset: 31 days.
		\end{itemize}
		\item Problem 2 (SF): Learning the arrival rate of San Francisco taxi requests.
		\begin{itemize}
			\item Data: Dataset of mobility traces of taxi cabs in San Francisco, US. It contains GPS coordinates of approximately 564 taxis collected over 31 days. The average number of requests per day is around 14.255. The time calls are specified to be the pickup time subtracting 10 minutes and clustered every 5 minutes.
			\item Link for the dataset: \emph{http://cabspotting.org/}
			\item Training dataset: The data collected over 24 days in 03/2010.
			\item Testing dataset: The data  collected over 7 days in 03/2010.
		\end{itemize}
		\item Problem 3 (PT): Learning the arrival rate of Porto taxi requests.
		\begin{itemize}
			\item Data: An accurate dataset describes a complete year of the trajectories for all of the 442 taxis running in the city of Porto, Portugal. These taxis operate through a taxi dispatch central and serve 1710589 trips in a year from 01/07/2013.
			\item Link for the dataset: \emph{https://www.kaggle.com}
			\item Training dataset: The data  collected from from 01/07/2013 to 30/05/2014.
			\item Testing dataset: The data  collected over 30 days in 06/2014.
		\end{itemize}
		\item Problem 4 (IB): Learning the arrival rate of calls to a call-center at a small Israeli bank.
		\begin{itemize}
			\item Data: detailed description of the data could be found in  \cite{Mandelbaum2001}. There are 135780 calls arrived in four months. Each record includes the time stamp for when the call arrived. The average number of arrival calls per minute is 1.1.
			\item Link for the dataset: \emph{http://ie.technion.ac.il/serveng/callcenterdata/inde}-\emph{x.html}
			\item Training dataset: The data  collected from from 01/01/1999 to 30/04/1999.
			\item Testing dataset: The data collected over 30 days in 05/1999.
		\end{itemize}
	\end{itemize}
	In this paper, we conduct three main experiments. 
	\subsubsection{Experiment 1}
	\begin{itemize}
		\item \emph{Description}: The first experiment is to confirm that the learning with data binning on  finite data can carry the probability of being overfitting.
		\item \emph{Data}: The SF, PT, IB, and AF datasets are used.
	\end{itemize}
	\subsubsection{Experiment 2}
	\begin{itemize}
		\item \emph{Description}: In the second experiment, we perform our proposed algorithms to evaluate the efficiency of learning regularization for NHPPs.
		\item \emph{Data}: The SF, PT, IB and AF datasets are used.
	\end{itemize}
	\subsubsection{Experiment 3}
	\begin{itemize}
		\item \emph{Description}: In the last experiment, we test our proposed algorithms for a more general case in real life, i.e., Dynamic Vehicle Routing Problem (DVRP) with spatio-temporal data.
		\item \emph{Data}: The SF dataset is used.
	\end{itemize}

	\subsection{Training and Testing Processes}
	\label{expproblemtraining}
	In our experiments, the range of considered time $[0, T]$ is set by day unit. The arrival rate functions depend on the time of day (i.e., we assume that requests of different dates come from the same distribution). Therefore, the number of arrival events is a result of repeated day-to-day observations. For each day in the training dataset, the number of requests coming at a time of day (the arrival time is rounded by minute) is considered as a data point. 
	To avoid an ecological fallacy, all those data points are used to train the model (see \citealp{Garrett2002}). As illustrated in Fig. \ref{figure-traintestprocess}, the number of values on the vertical axis (possibly equal) at each value on the horizontal axis is equal to the number of days in the considered dataset. The training process minimizes the \emph{Root Mean Squared Error} (RMSE) on all those data points. The testing process is performed similarly on the test dataset.
	\begin{figure}
		\centering
		\includegraphics[width = \linewidth]{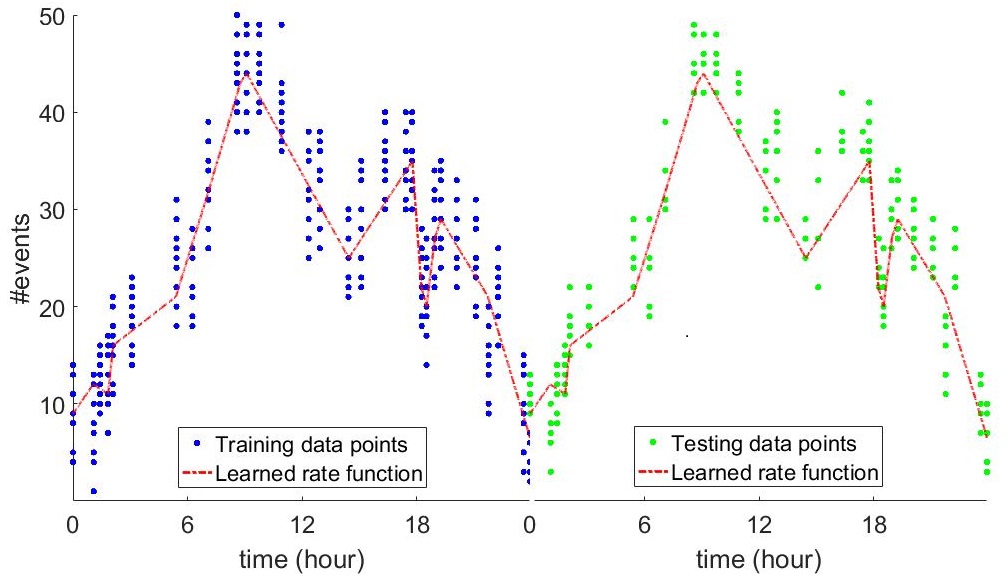}
		\caption{An example of the training and testing processes}
		\label{figure-traintestprocess}
	\end{figure}
	\subsection{Experiment 1: Overfitting Problem}
	To show the possibility of being overfitting when using data binning, we conducted the experiments of learning NHPPs with increasing number of bins. It could simply be done by relaxing the constraint that requires the Poisson property in Algorithm \ref{sim:dividedGeneralAlgo}. In line 6, if either of the two datasets $O_1, O_2$ does not pass the distribution test, and its length is not bigger than $2\eta$ (minute unit), the division on dataset $O$ is not performed. This ensures that the number of bins is finite as common in the literature. The piecewise-polynomial regressions with different $\eta$ were performed. The RMSE training ($\sigma_1$) and testing ($\sigma_2$) errors  for all tested problems are presented in Fig. \ref{experiments-Figure-overfitting}. Table \ref{experiments-Table-overfitting} shows the detailed numerical results on the SF dataset. 
	\begin{table}[h]
		\caption{Results of the piecewise-polynomial regression with the relaxed division algorithm.}
		\label{experiments-Table-overfitting}
		\begin{tabular*}{.7\textwidth}{@{\extracolsep\fill}lccc}
				\toprule
				Ins	& $\#$b & $\sigma_1$ & $\sigma_2$\\
				\midrule
				$\eta$ = 600 & 3 & 6.54 & 6.3\\
				\midrule
				$\eta$ = 480 & 4 & 6.43 & 6.23\\
				\midrule
				$\eta$ = 120 & 7 & 5.23 & 5.23\\
				\midrule
				$\eta$ = 100 & 8 & 5.21 & 5.2\\
				\midrule
				$\eta$ = 80 & 11 & 5.1 & 5.07\\
				\midrule
				$\eta$ = 60 & 14 & 4.91 & 5.21\\
				\midrule
				$\eta$ = 50 & 15 & 4.9 & 5.19\\
				\midrule
				$\eta$ = 40 & 19 & 4.88 & 5.21\\
				\midrule
				$\eta$ = 30 & 26 & 4.76 & 5.24\\
				\midrule
				$\eta$ = 20 & 48 & 4.38 & 5.31\\
				\midrule
				$\eta$ = 10 & 79 & 4.26 & 5.48\\
				\bottomrule
			\end{tabular*}
	\end{table}
	\begin{figure}
		\centering
		\includegraphics[width = \linewidth]{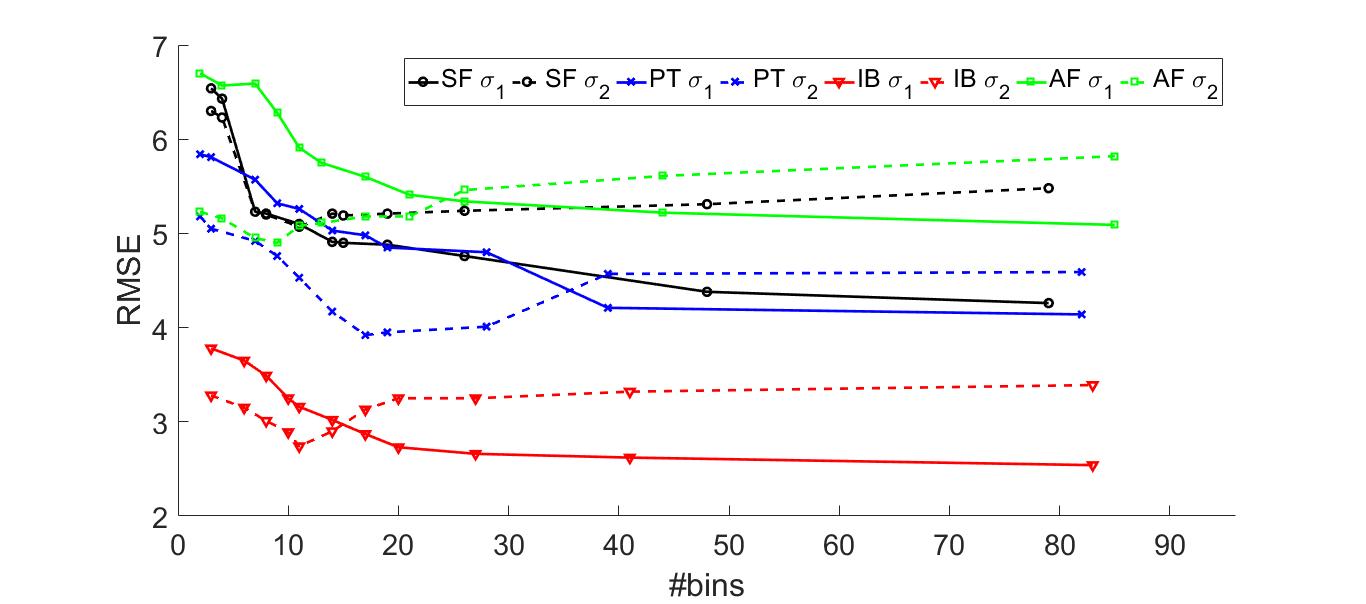}
		\caption{Overfitting problem in piecewise-polynomial regression}
		\label{experiments-Figure-overfitting}
	\end{figure}
	It can be seen that when the number of subintervals ($\#$b) gets increased, initially, both $\sigma_1$ and $\sigma_2$ errors steadily decrease. However, starting from 14 bins, the training error decreases slowly and the test error gets increased faster. This behavior is commonly attributed to the overfitting process in the literature of machine learning. The same is true for all other tested problems (see Table \ref{experiments-Table-overfitting-Apendix} in the Appendix) and it is consistent with the theoretical result shown in subsection \ref{LearningG}.
	
	
	
	\subsection{Experiment 2: Learning Regularization of NHPPs}
	We denote the RMSE on the training and the test datasets as $\sigma_1$ and $\sigma_2$, respectively. The number of subintervals found by Algorithm \ref{sim:binning} is denoted as $\#b$. For comparison, both regularized data binning method based on Ivanov ($\textit{DBM}_I$) and Tikhonov ($\textit{DBM}_T$) regularization were performed. We also conducted the common data binning method in the literature, where the considered interval is divided into $\#b$ subintervals with equal length. $\overline{\sigma_1}$ and $\overline{\sigma_2}$ denote the RMSEs on the training and test datasets found by the estimation process. Notation $\rho$ presents the improvement percentage of $\sigma_2$ found by the regression on subintervals with our binning method compared to $\overline{\sigma_2}$ found by the regression on subintervals with equal length.
	\subsubsection{The piecewise-linear and linear regression}
	Table \ref{experiments2-linear} shows the results of learning by linear (\emph{Unbinned} column) and piecewise-linear functions. It is undoubted that binning is essentially important for learning NHPPs with both training and test errors are significantly lower on all tested problems as shown in the table. The results also show that adaptive binning with regularization helps to improve on test errors from 10.73\% to 25.57\% compared to equal length binning. For the two regularization methods, Ivanov regularization ($\textit{DBM}_I$) gives better results than Tikhonov regularization ($\textit{DBM}_T$) on the AF, PT and IB problems and worse on SF. The result illustrations are placed in the Appendix for the convenience of reference (Figures \ref{Apendix-figure-regressionAFD-linear}-\ref{Apendix-figure-regressionIB-poly}).
	
	
	\begin{table}
		\caption{The regression results of the rate function approximated by linear and piecewise-linear functions.}
		\label{experiments2-linear}
		\begin{tabular*}{\textwidth}{@{\extracolsep\fill}|c|c|c|c|c|c|c|c|c|c|c|c|c|c|c|}
			\cmidrule{1-15}
				\multirow{2}{*}{Ins}	& \multicolumn{2}{@{}c|@{}}{Unbinned} & \multicolumn{6}{@{}c|@{}}{Binning with $\textit{DBM}_I$} & \multicolumn{6}{@{}c|@{}}{Binning with $\textit{DBM}_T$}\\
				\cmidrule{2-15} 
				& $\sigma_1$ & $\sigma_2$ & $\sigma_1$ & $\sigma_2$ & $\#b$ & $\overline{\sigma_1}$ & $\overline{\sigma_2}$ & $\rho$ & $\sigma_1$ & $\sigma_2$ & $\#b$ & $\overline{\sigma_1}$ & $\overline{\sigma_2}$ & $\rho$\\
				\cmidrule{1-15}
				AF	&	8.89	&	10.6	&	4.18	&	\textbf{4.74}	&	10	&	5.17	&	5.31	&	10.73	&	4.75	&	4.86	&	12	&	5.21	&	5.46	&	10.98	\\
				SF	&	11.1	&	13.4	&	3.2	&	4.27	&	12	&	5.02	&	5.14	&	16.93	&	3.16	&	\textbf{3.87}	&	13	&	4.96	&	5.2	&	25.57	\\
				PT	&	5.43	&	5.69	&	2.82	&	\textbf{3.28}	&	15	&	4.73	&	4.05	&	19.01	&	2.96	&	3.84	&	11	&	4.92	&	4.47	&	14.09	\\
				IB	&	5.74	&	6.06	&	2.42	&	\textbf{2.73}	&	11	&	2.78	&	3.19	&	14.42	&	2.41	&	2.81	&	9	&	3.12	&	3.36	&	16.36	\\
				\cmidrule{1-15}
			\end{tabular*}
	\end{table}
	
	\subsubsection{The polynomial and piecewise-polynomial regression}
	In this part, the hypothesis spaces are polynomial (unbinned) and piecewise-polynomial functions (binning). Table \ref{table:nonlinearReg} presents the quality improvement of algorithms using regularizations $\textit{DBM}_I$ and $\textit{DBM}_T$. 
	\begin{table}
		\caption{The regression results of the rate function approximated by polynomial and piecewise-polynomial functions.}
		\label{table:nonlinearReg}
			\begin{tabular*}{\textwidth}{@{\extracolsep\fill}|c|c|c|c|c|c|c|c|c|c|c|c|c|c|c|}
				\cmidrule{1-15}
				\multirow{2}{*}{Ins}	& \multicolumn{2}{@{}c|@{}}{Unbinned} & \multicolumn{6}{@{}c|@{}}{Binning with $\textit{DBM}_I$} & \multicolumn{6}{@{}c|@{}}{Binning with $\textit{DBM}_T$}\\
				\cmidrule{2-15} 
				& $\sigma_1$ & $\sigma_2$ & $\sigma_1$ & $\sigma_2$ & $\#$b & $\overline{\sigma_1}$ & $\overline{\sigma_2}$ & $\rho$ & $\sigma_1$ & $\sigma_2$ & $\#$b & $\overline{\sigma_1}$ & $\overline{\sigma_2}$ & $\rho$\\
				\cmidrule{1-15}
				AF	&	5.8	&	5.97	&	4.84	&	\textbf{4.91}	&	8	&	5.33	&	5.38	&	8.74	&	4.78	&	4.96	&	12	&	5.27	&	5.49	&	9.65	\\
				SF	&	3.69	&	4.41	&	3.02	&	\textbf{3.56}	&	15	&	3.34	&	4.33	&	17.78	&	2.34	&	3.69	&	14	&	3.41	&	4.38	&	15.75	\\
				PT	&	3.15	&	4.72	&	1.91	&	\textbf{3.9}	&	20	&	3.02	&	4.62	&	15.58	&	1.91	&	3.98	&	12	&	2.88	&	4.53	&	12.14	\\
				IB	&	2.39	&	3.38	&	1.98	&	\textbf{2.63}	&	16	&	2.17	&	3.12	&	15.71	&	2.06	&	2.69	&	13	&	2.24	&	2.84	&	5.28	\\
				
				\cmidrule{1-15}
			\end{tabular*}
	\end{table}
	The results again confirm the importance of data binning and the regularization approaches. $\textit{DBM}_I$ has the best test errors on all tested problems improving over those of equal length binning method from 8.74\% to 17.78\%. Comparing with the results in Table \ref{experiments2-linear}, 
	the piecewise-polynomial regression on PT and AF datasets are worse the piecewise-linear regression. A natural explanation is that the rate functions of these datasets come from piecewise-linear functions (certainly for AF) or the choice of the degree of polynomial (PT). 
	
	Overall, the experimental results clearly indicate the effectiveness of regularization methods over the vanilla equal length binning method.
	Of the two methods, $\textit{DBM}_I$ seems to be the best. However, $\textit{DBM}_T$ might still be useful since it has lower computational time (without repeatedly performing statistical tests) and for the case of high dimensional NHPPs, where the efficient statistical tests are not available.
	
	
	\subsection{Experiment 3: Dynamic Vehicle Routing Problem (DVRP) with Spatio-Temporal Data}
	\begin{figure}
		\centering
		\includegraphics[width = 3in]{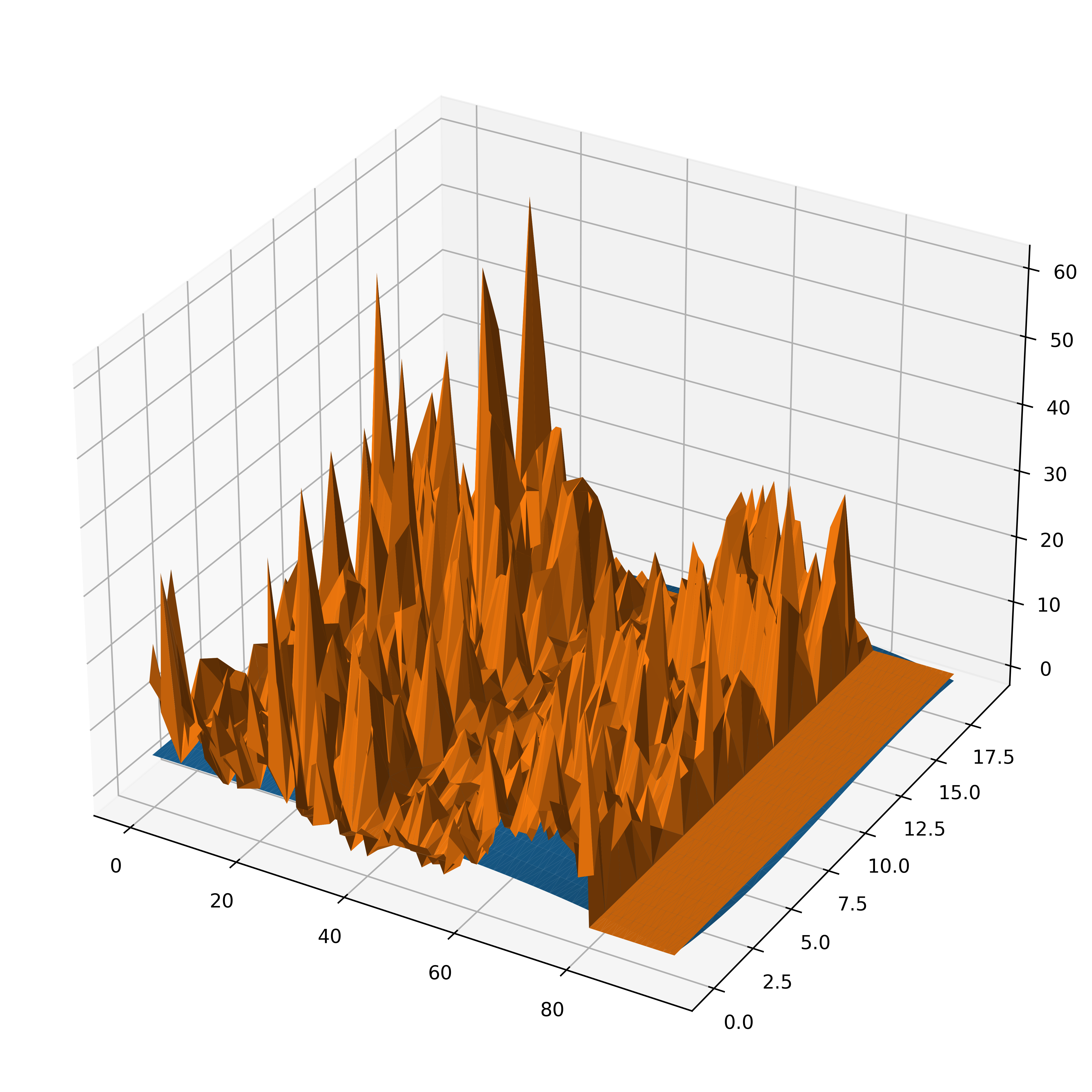}
		\caption{The result on the SF dataset with 20 sub-areas}
		\label{figure-poly-96-20-SF}
	\end{figure}
	\begin{table}
		\caption{The profit of scheduling algorithm using our proposed learning method.}
		\label{experiments:applicationTable}
		\begin{tabular*}{.7\textwidth}{@{\extracolsep\fill}lccc}
			\toprule
				Ins	&	$p_1$    &	$p_2$ & $\rho (\%)$\\
				\midrule
				Day	1	&	27.44	&	\textbf{29.14}	&	6.19	\\
				Day	2	&	30.39	&	\textbf{32.11}	&	5.66	\\
				Day	3	&	31.71	&	\textbf{33.98}	&	7.16	\\
				Day	4	&	29.62	&	\textbf{31.03}	&	4.76	\\
				Day	5	&	29.8	&	\textbf{31.61}	&	6.07	\\
				Day 6   &   29.24   &   \textbf{30.15}   &   3.11    \\
				Day 7   &   29.62   &   \textbf{31.17}   &   5.23    \\
				Day 8   &   28.56   &   \textbf{29.72}   &   4.06    \\
				Day 9   &   10.16   &   \textbf{10.85}   &   6.79    \\
				\bottomrule
			\end{tabular*}
		
	\end{table}
	In this experiment, we evaluate the benefit of using our regularized learning of NHPPs for a recent DVRP application. In the DVRP, a spatio-temporal process can be reduced to temporal processes by a clustering algorithm as proposed in \cite{Bent2003}. A large geographical area can be clustered into non overlap sub-areas by the $K$-means clustering algorithm. Most of recent works on DVRP have utilized this approach, for instance, \cite{Ferrucci2013, SonNV}. In \cite{SonNV}, the authors consider the share-a-ride taxi scheduling problem and assume that the arrival rate of taxi requests only depend on period and area and hence are fixed overall days. The city is partitioned into a set of sub-areas. Each day is binned into subintervals with equal lengths. The number of arrival requests on each sub-area in a subinterval is generated by a Poisson distribution. In this experiment, we change the learning method in \cite{SonNV} by our proposed learning process using the regularization method $DBM_I$ to learn the arrival rate of taxi requests. In our experiment, the arrival rate of taxi requests is learned and used to suggest the best route from the last drop-off point to a parking lot for each taxi. Fig. \ref{figure-poly-96-20-SF} relates the piecewise-polynomial regression result on the SF dataset in which the number of sub-areas is equal to 20 ($K = 20$). Table \ref{experiments:applicationTable} shows the profit of scheduling algorithms. We denote the $\textit{OTSF-DP}$ algorithm in \cite{SonNV} by $p_1$ and the $\textit{OTSF-DP}$ using our new predicted rate function by $p_2$. The percentage of improvement is denoted by $\rho = (p_2 - p_1)100/p_1$. It can be seen that by using the new predicted rate function obtained via regularized learning approach in the paper, the profit is improved significantly for all nine tested days. The percentage of improvement is from 3.11\% (909 USD per day) to 7.16\% (2270 USD per day), which is of significantly economical values.
	
	\section{Conclusion} \label{conclusion}
	In this paper, we formulate the problem of estimation of NHPPs from finite and limited data as a machine learning problem. We have theoretically and experimentally shown that while the binning of observed data is essential for learning/estimating NHPPs, it could lead to overfitting. Consequently, we have proposed a learning framework of NHPPs that utilizes two regularization methods based on ideas of the Tikhonov \citep{Tikhonov} and Ivanov regularization \citep{Ivanov}. We tested our regularization learning methods on synthetic and real-world datasets and the experimental results show the effectiveness of these methods. The learning methods were further tested by embedding their results in a recent real-world application of share-a-ride taxi scheduling using spatio-temporal data. Our new regularized learning methods help to significantly improve the performance of one of the state-of-the-art dynamic scheduling algorithms producing better economical profits. Moreover, the new regularized learning methods completely remove the need of ad-hoc parameter tuning in binning (such as the number of bins) by a data driven binning method.

	In future, we will extend our approach to learning high dimensional NHPPs and extend the VC-theory to fully cover the statistical learning of NHPPs.

\backmatter

\section*{Declarations}

\begin{itemize}
\item Funding: The authors did not receive support from any organization for the submitted work
\item Conflict of interest/Competing interests: The authors have no competing interests to declare that are relevant to the content of this article.
\item Ethics approval: Not applicable
\item Consent to participate: Not applicable
\item Consent for publication: Not applicable
\item Availability of data and materials: The instances and our results are readily available on \emph{https://github.com/sonnv188/poissonprocess.git}
\item Code availability: The instances and our results are readily available on \emph{https://github.com/sonnv188/poissonprocess.git}
\item Authors' contributions: All authors contributed to the study conception and design. Conceptualization, methodology and validation was proposed by Hoai Nguyen Xuan. Material preparation, implementing and data analysis were performed by Son Nguyen Van. The first draft of the manuscript was written by Son Nguyen Van and all authors commented on previous versions of the manuscript. All authors read and approved the final manuscript.
\end{itemize}

\begin{appendices}

\section{Learning Results}\label{secA1}
\begin{table}
		\caption{Results of the piecewise-polynomial regression with the relaxed division algorithm}
		\label{experiments-Table-overfitting-Apendix}
		\begin{tabular*}{\textwidth}{@{\extracolsep\fill}|c|c|c|c|c|c|c|c|c|c|}
			\cmidrule{1-10}
			\multirow{2}{*}{Ins} & \multicolumn{3}{c|}{AF} & \multicolumn{3}{c|}{PT} & \multicolumn{3}{c|}{IB}\\
			\midrule{2-10}
			& $\#$b & $\sigma_1$ & $\sigma_2$ & $\#$b & $\sigma_1$ & $\sigma_2$ & $\#$b & $\sigma_1$ & $\sigma_2$\\
			\hline
			$\eta$	=	600	&	2	&	6.7	&	5.23	&	2	&	5.84	&	5.18	&	3	&	3.78	&	3.28	\\
			$\eta$	=	480	&	4	&	6.57	&	5.16	&	3	&	5.81	&	5.05	&	6	&	3.65	&	3.15	\\
			$\eta$	=	120	&	7	&	6.59	&	4.95	&	7	&	5.57	&	4.92	&	8	&	3.49	&	3.01	\\
			$\eta$	=	100	&	9	&	6.28	&	4.9	&	9	&	5.32	&	4.76	&	10	&	3.25	&	2.89	\\
			$\eta$	=	80	&	11	&	5.91	&	5.08	&	11	&	5.26	&	4.53	&	11	&	3.16	&	2.74	\\
			$\eta$	=	60	&	13	&	5.75	&	5.11	&	14	&	5.03	&	4.17	&	14	&	3.02	&	2.9	\\
			$\eta$	=	50	&	17	&	5.6	&	5.18	&	17	&	4.98	&	3.92	&	17	&	2.87	&	3.13	\\
			$\eta$	=	40	&	21	&	5.41	&	5.18	&	19	&	4.85	&	3.95	&	20	&	2.73	&	3.25	\\
			$\eta$	=	30	&	26	&	5.34	&	5.46	&	28	&	4.8	&	4.01	&	27	&	2.66	&	3.25	\\
			$\eta$	=	20	&	44	&	5.22	&	5.61	&	39	&	4.21	&	4.57	&	41	&	2.62	&	3.32	\\
			$\eta$	=	10	&	85	&	5.09	&	5.82	&	82	&	4.14	&	4.59	&	83	&	2.54	&	3.39		\\
			\cmidrule{1-10}
		\end{tabular*}
\end{table}

\begin{figure}[H]
	\centering
	\subfloat[$\textit{DBM}_I$]{\includegraphics[width=0.5\textwidth]{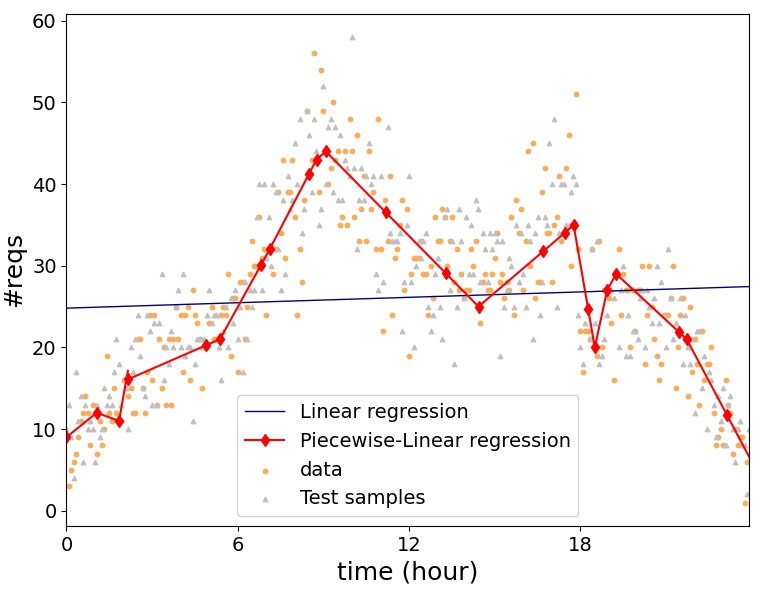}%
		\label{Apendix-figure-regressionAFD-linear-Ivanov}}
	\hfil
	\subfloat[$\textit{DBM}_T$]{\includegraphics[width=0.5\textwidth]{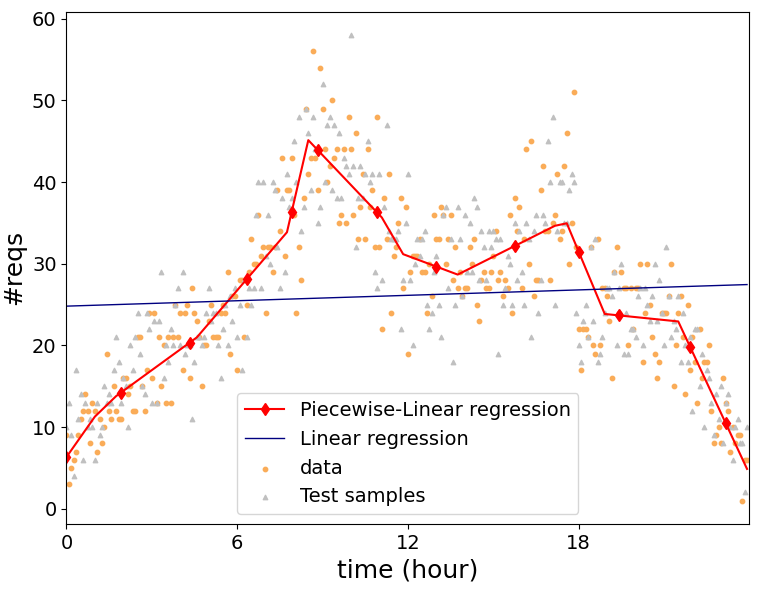}%
		\label{Apendix-figure-regressionAFD-linear-Tikhonov}}
	\caption{The piecewise-linear regression of the arrival rate function on the AF dataset}
	\label{Apendix-figure-regressionAFD-linear}
\end{figure}

\begin{figure}[H]
	\centering
	\subfloat[$\textit{DBM}_I$]{\includegraphics[width=0.5\textwidth]{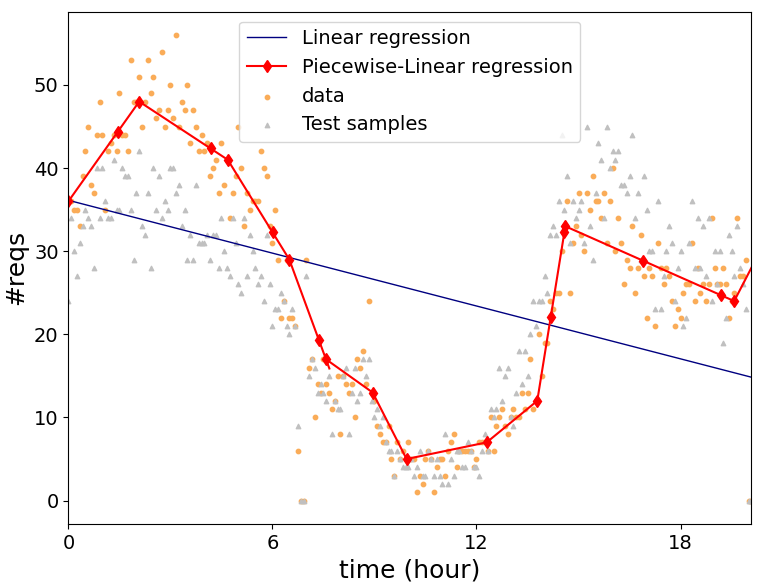}%
		\label{Apendix-figure-regressionSF-linear-Ivanov}}
	\hfil
	\subfloat[$\textit{DBM}_T$]{\includegraphics[width=0.5\textwidth]{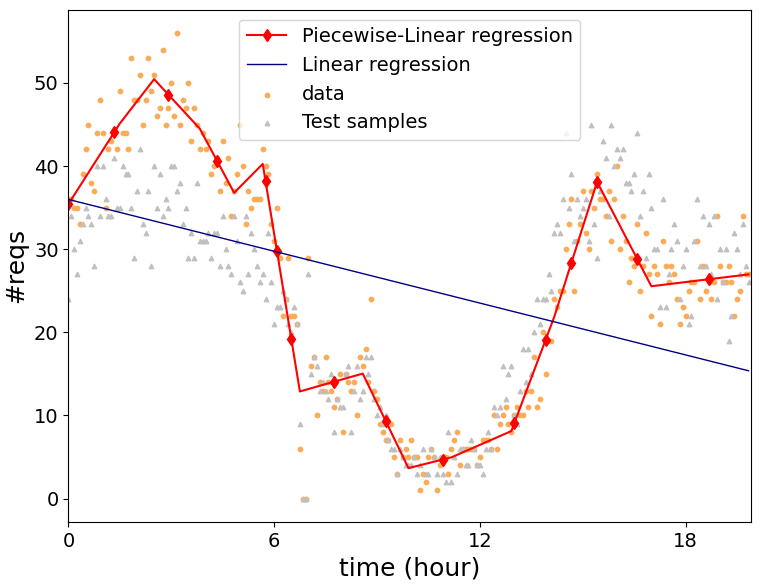}%
		\label{Apendix-figure-regressionSF-linear-Tikhonov}}
	\caption{The piecewise-linear regression of the arrival rate function on the SF dataset}
	\label{Apendix-figure-regressionSF-linear}
\end{figure}

\begin{figure}[H]
	\centering
	\subfloat[$\textit{DBM}_I$]{\includegraphics[width=0.5\textwidth]{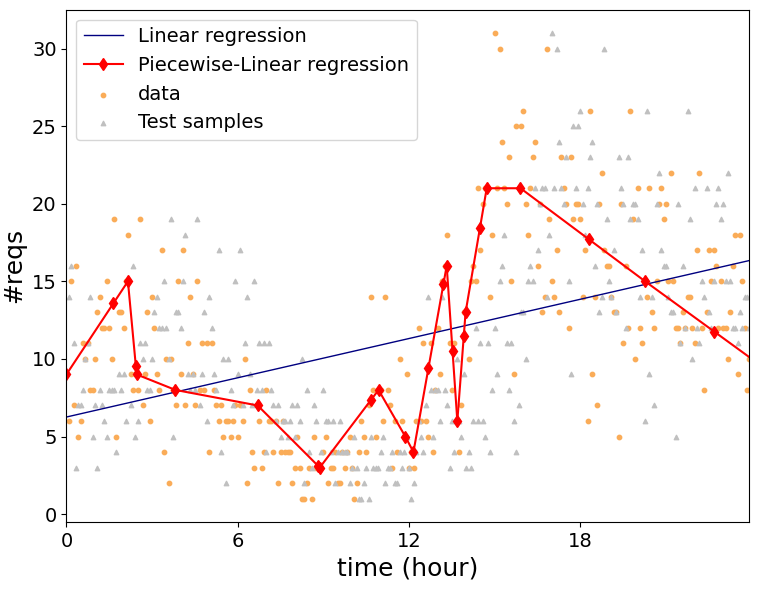}%
		\label{Apendix-figure-regressionPT-linear-Ivanov}}
	\hfil
	\subfloat[$\textit{DBM}_T$]{\includegraphics[width=0.5\textwidth]{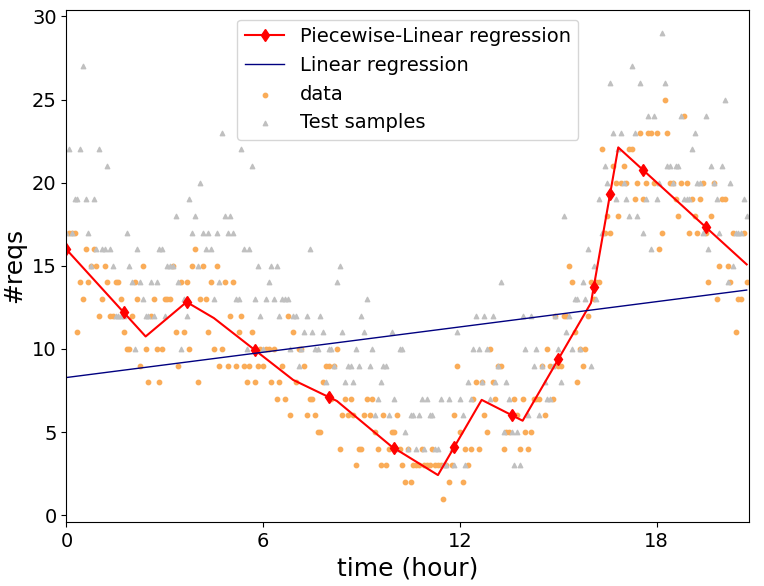}%
		\label{Apendix-figure-regressionPT-linear-Tikhonov}}
	\caption{The piecewise-linear regression of the arrival rate function on the PT dataset}
	\label{Apendix-figure-regressionPT-linear}
\end{figure}

\begin{figure}[H]
	\centering
	\subfloat[$\textit{DBM}_I$]{\includegraphics[width=0.5\textwidth]{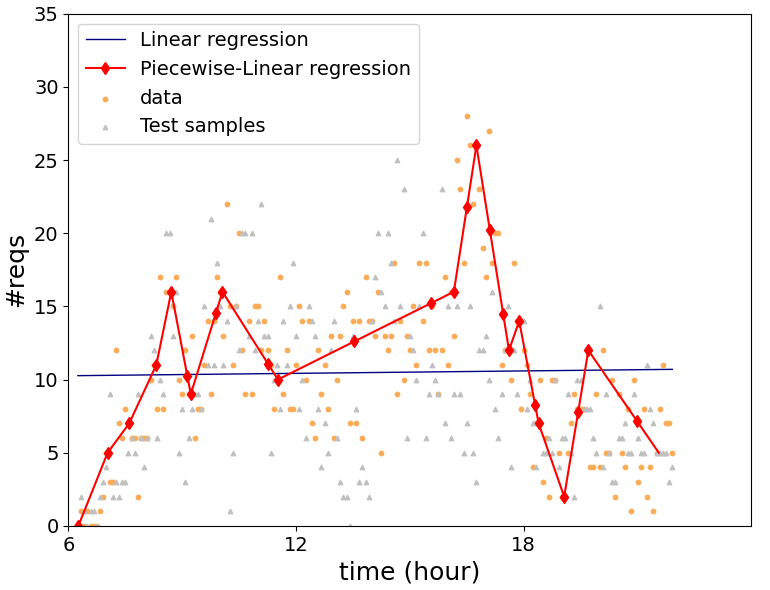}%
		\label{Apendix-figure-regressionIB-linear-Ivanov}}
	\hfil
	\subfloat[$\textit{DBM}_T$]{\includegraphics[width=0.5\textwidth]{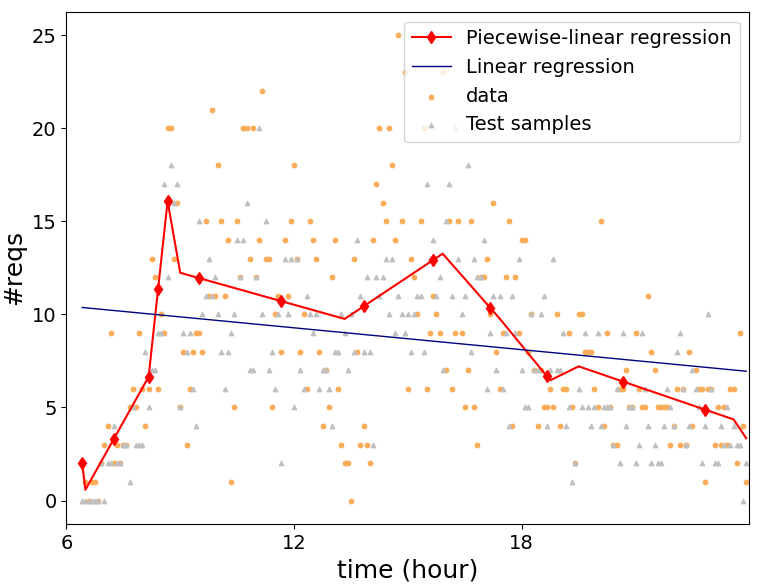}%
		\label{Apendix-figure-regressionIB-linear-Tikhonov}}
	\caption{The piecewise-linear regression of the arrival rate function on the IB dataset}
	\label{Apendix-figure-regressionIB-linear}
\end{figure}

\begin{figure}[H]
	\centering
	\subfloat[$\textit{DBM}_I$]{\includegraphics[width=0.5\textwidth]{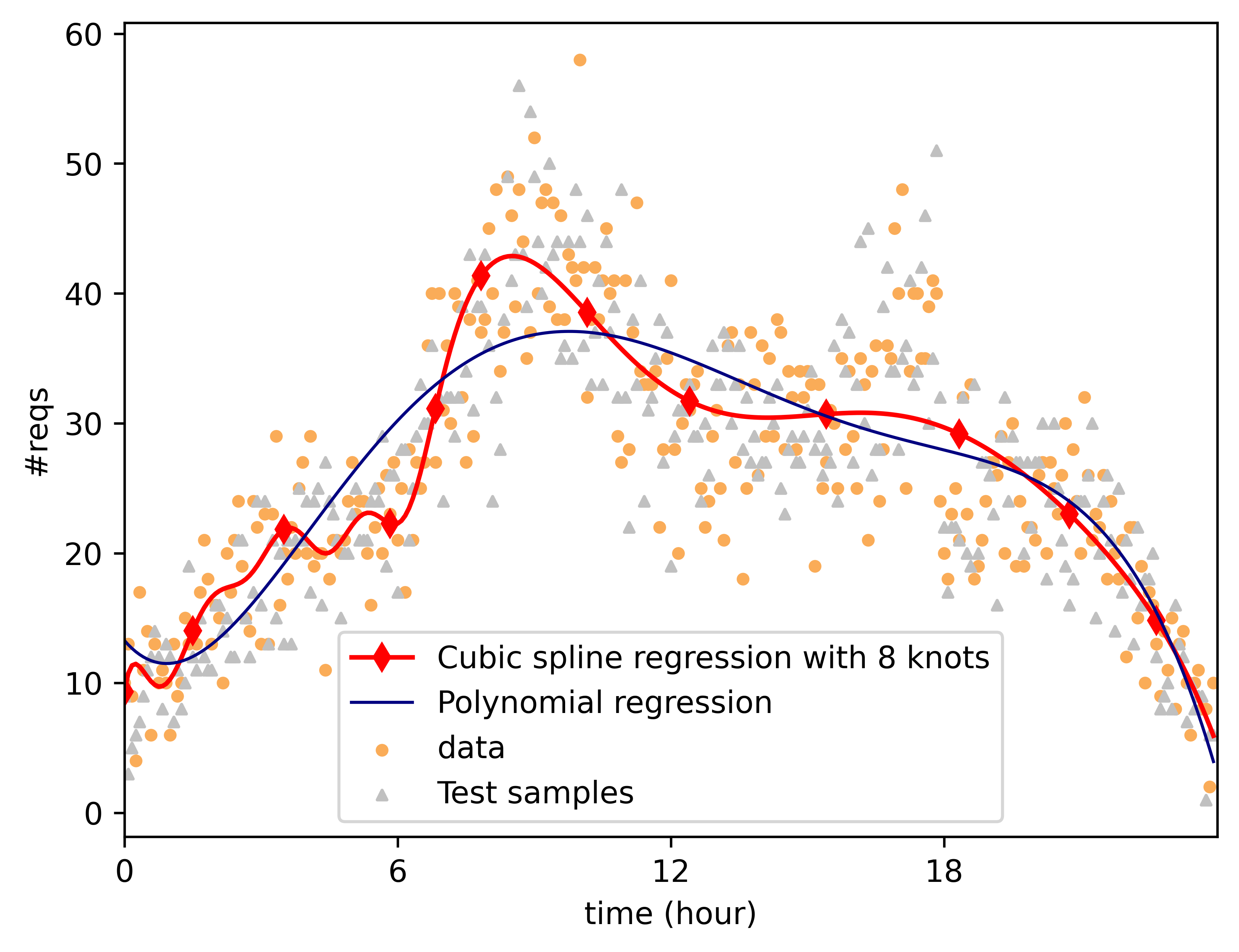}%
		\label{Apendix-figure-regressionAFD-poly-Ivanov}}
	\hfil
	\subfloat[$\textit{DBM}_T$]{\includegraphics[width=0.5\textwidth]{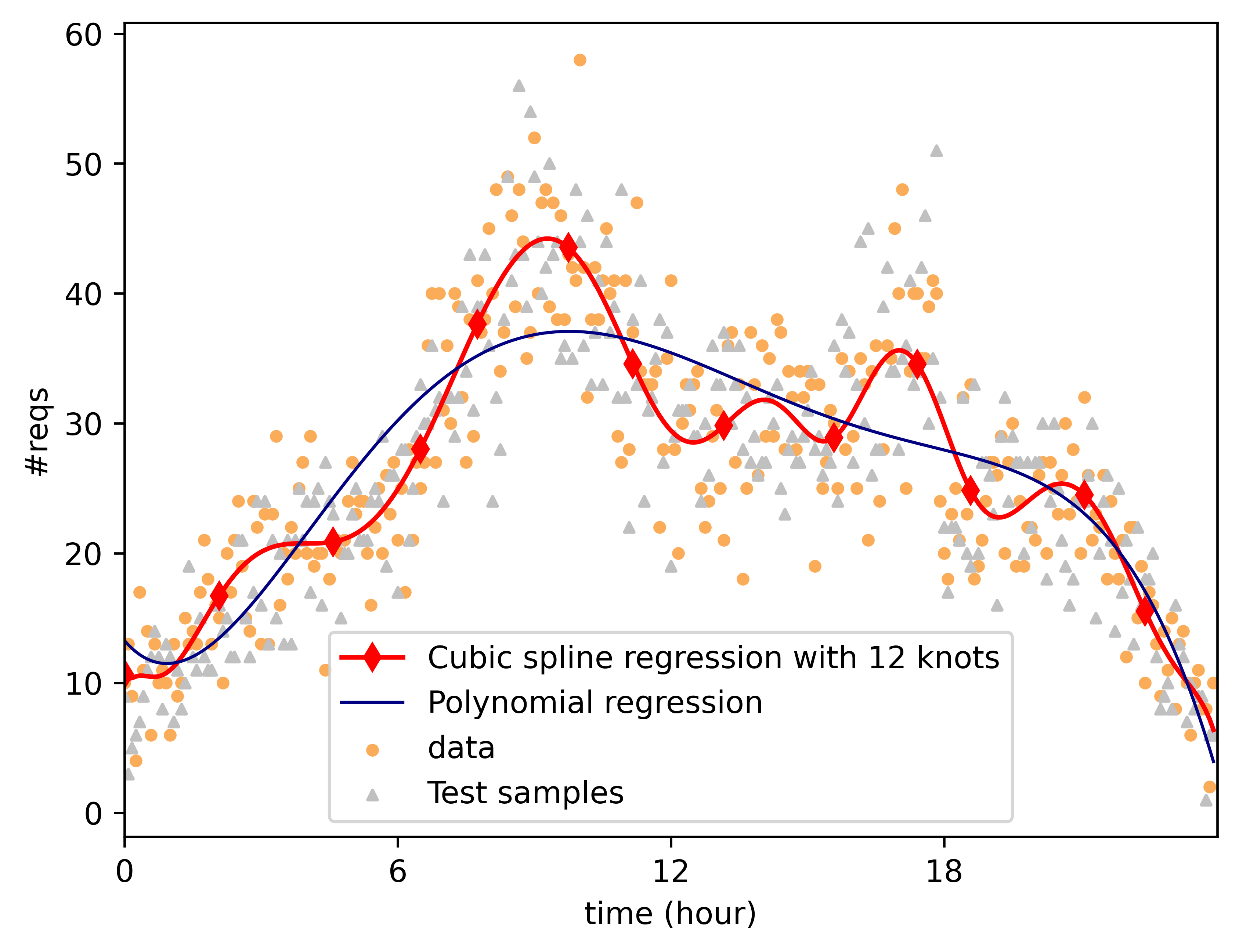}%
		\label{Apendix-figure-regressionAFD-poly-Tikhonov}}
	\caption{The piecewise-polynomial regression of the arrival rate function on the AF dataset}
	\label{Apendix-figure-regressionAFD-poly}
\end{figure}

\begin{figure}[H]
	\centering
	\subfloat[$\textit{DBM}_I$]{\includegraphics[width=0.5\textwidth]{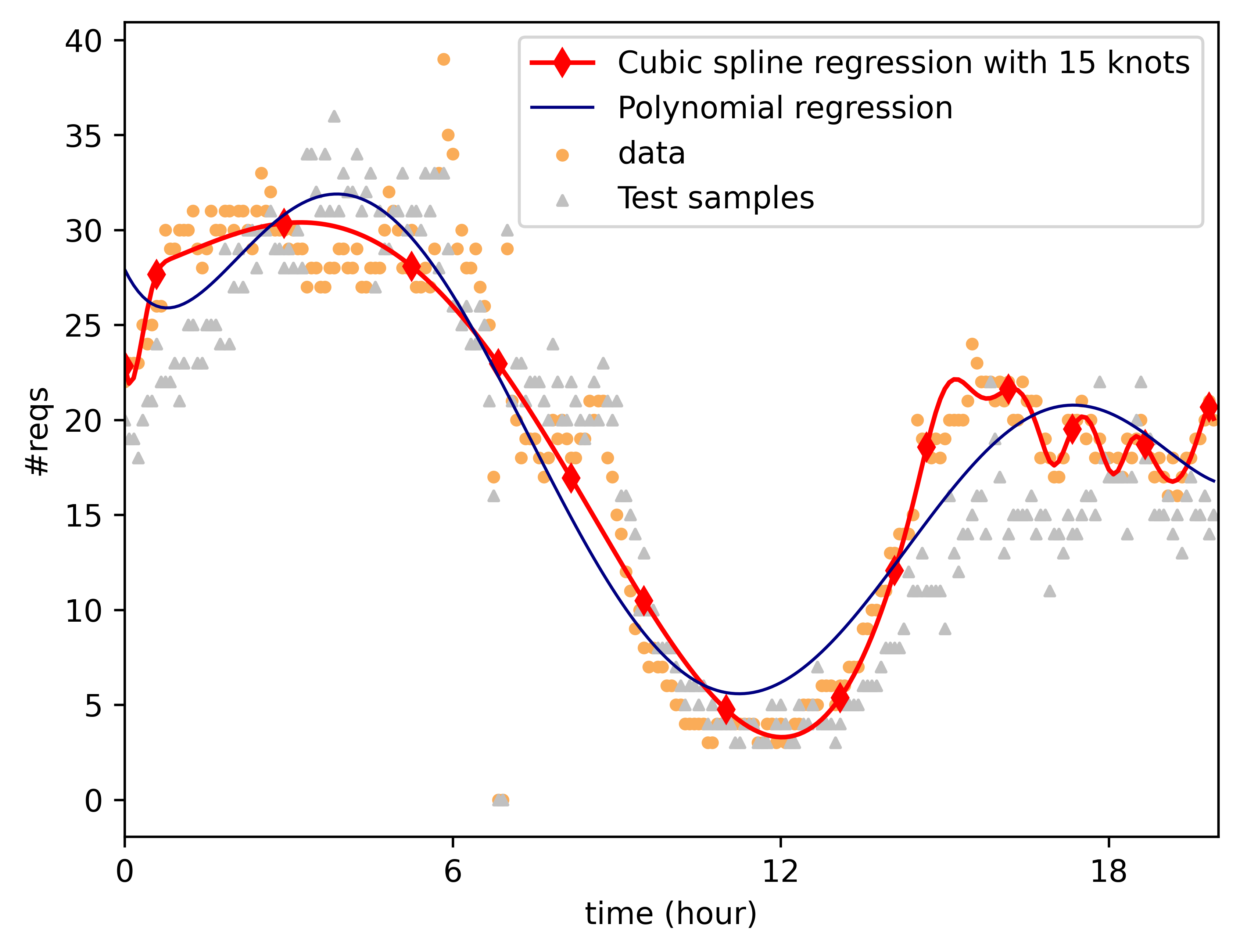}%
		\label{Apendix-figure-regressionSF-poly-Ivanov}}
	\hfil
	\subfloat[$\textit{DBM}_T$]{\includegraphics[width=0.5\textwidth]{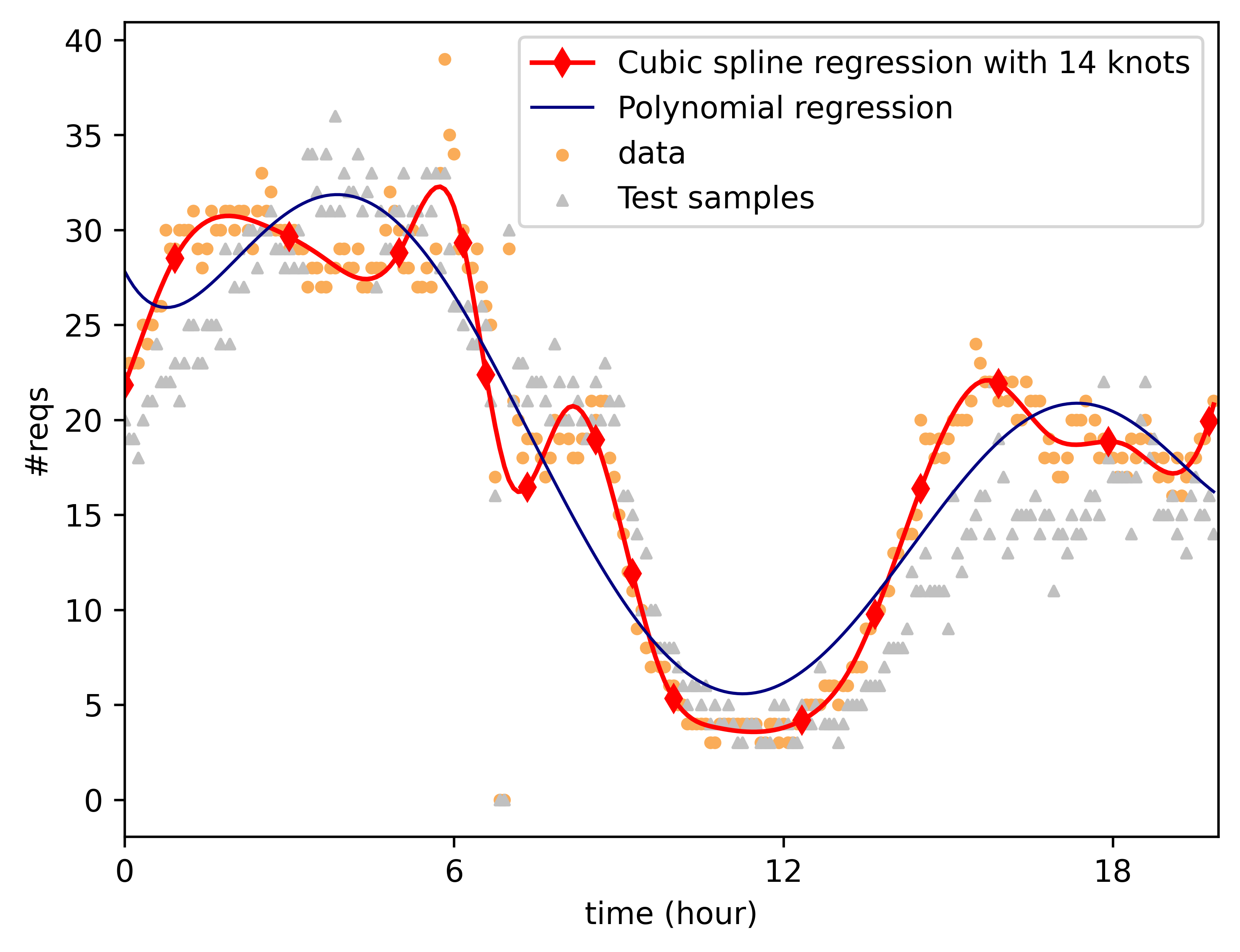}%
		\label{Apendix-figure-regressionSF-poly-Tikhonov}}
	\caption{The piecewise-polynomial regression of the arrival rate function on the SF dataset}
	\label{Apendix-figure-regressionSF-poly}
\end{figure}

\begin{figure}[H]
	\centering
	\subfloat[$\textit{DBM}_I$]{\includegraphics[width=0.5\textwidth]{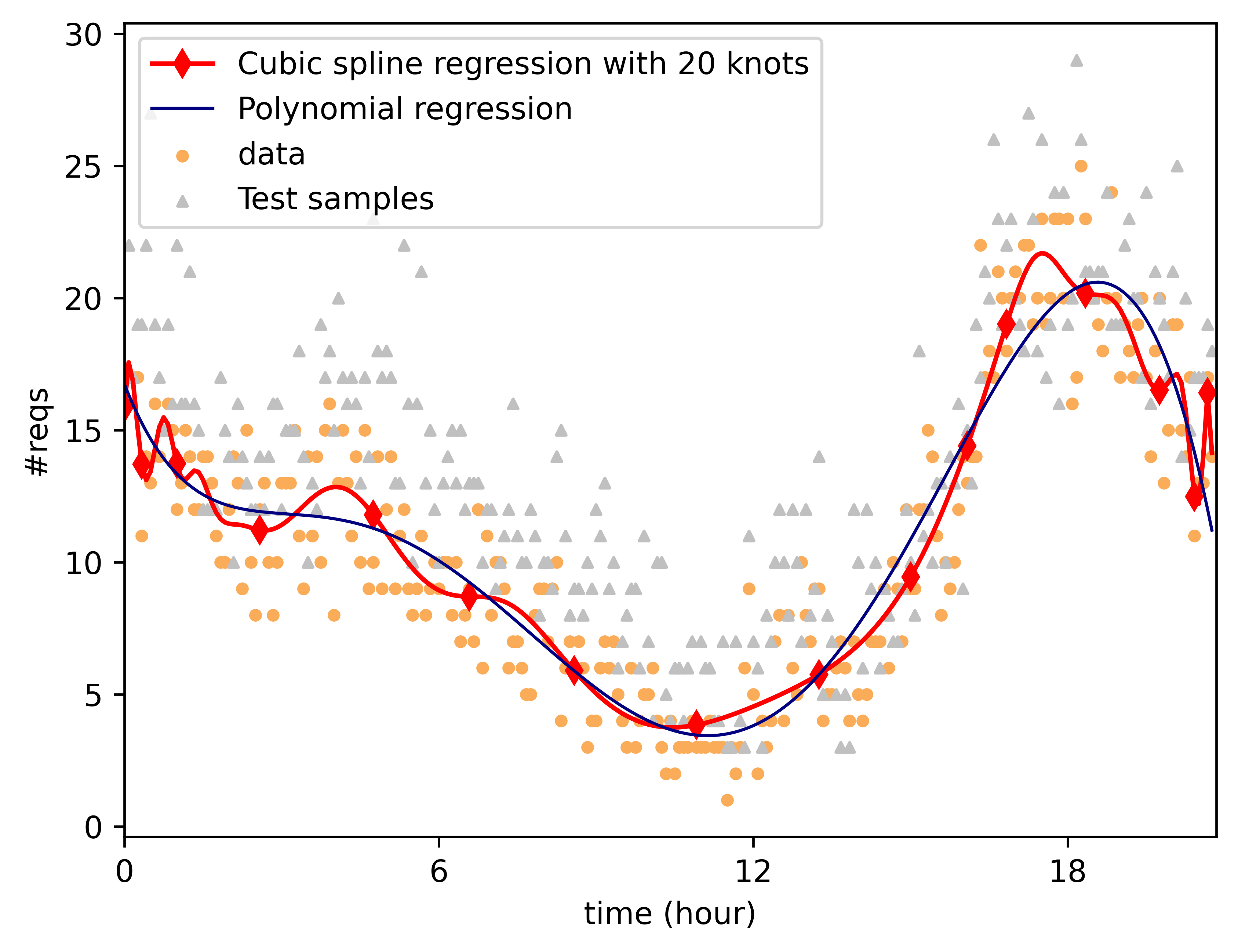}%
		\label{Apendix-figure-regressionPT-poly-Ivanov}}
	\hfil
	\subfloat[$\textit{DBM}_T$]{\includegraphics[width=0.5\textwidth]{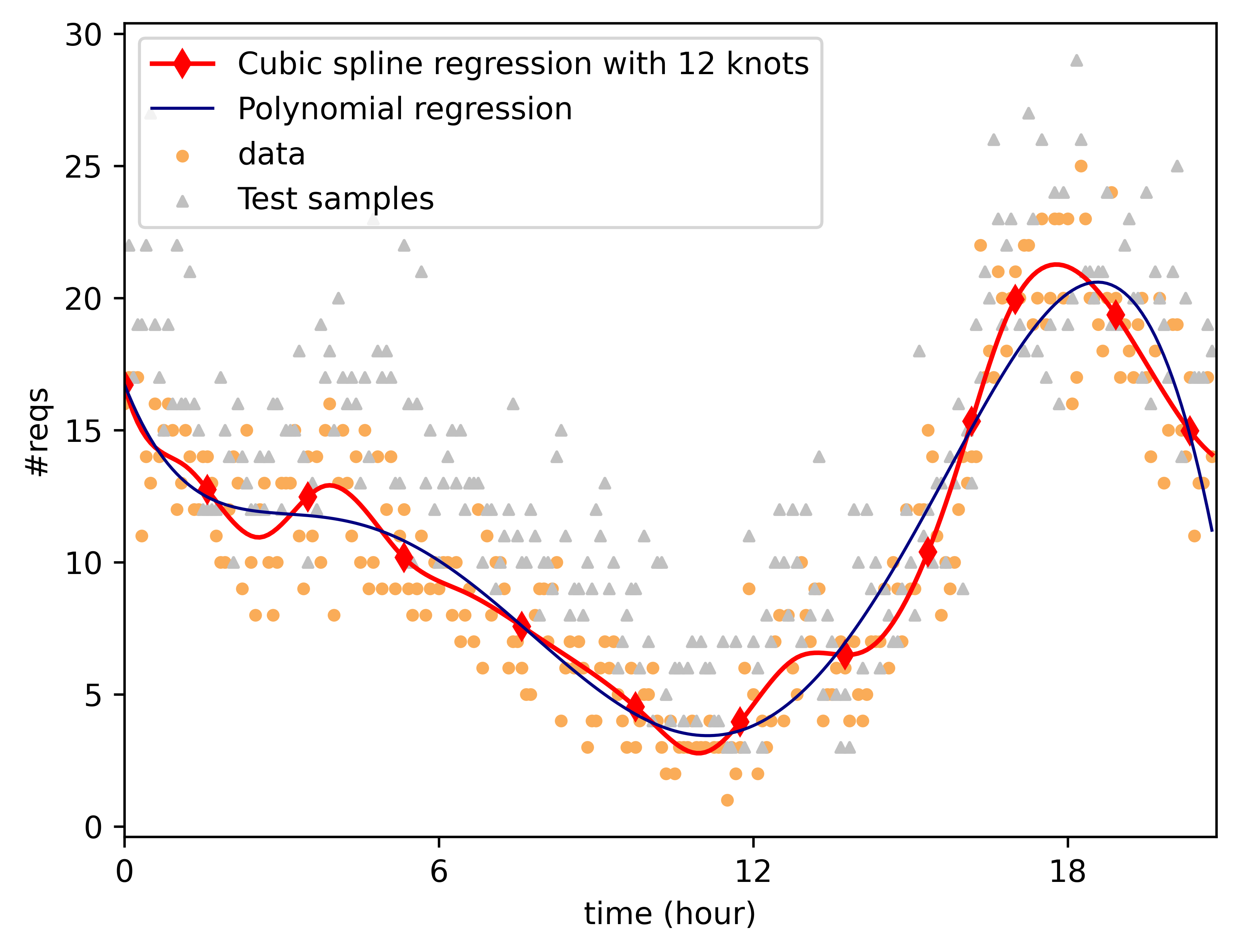}%
		\label{Apendix-figure-regressionPT-poly-Tikhonov}}
	\caption{The piecewise-polynomial regression of the arrival rate function on the PT dataset}
	\label{Apendix-figure-regressionPT-poly}
\end{figure}

\begin{figure}[H]
	\centering
	\subfloat[$\textit{DBM}_I$]{\includegraphics[width=0.5\textwidth]{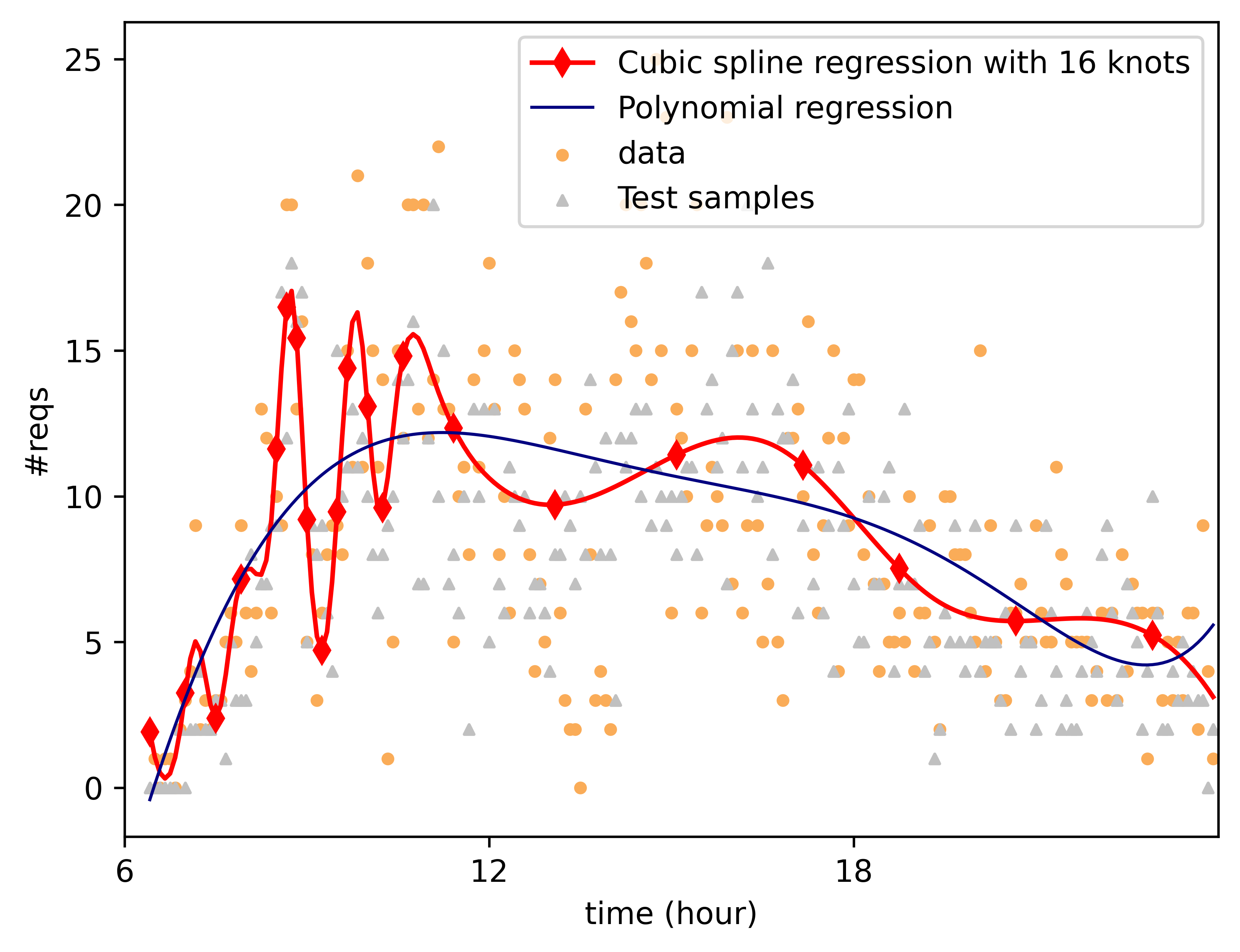}%
		\label{Apendix-figure-regressionIB-poly-Ivanov}}
	\hfil
	\subfloat[$\textit{DBM}_T$]{\includegraphics[width=0.5\textwidth]{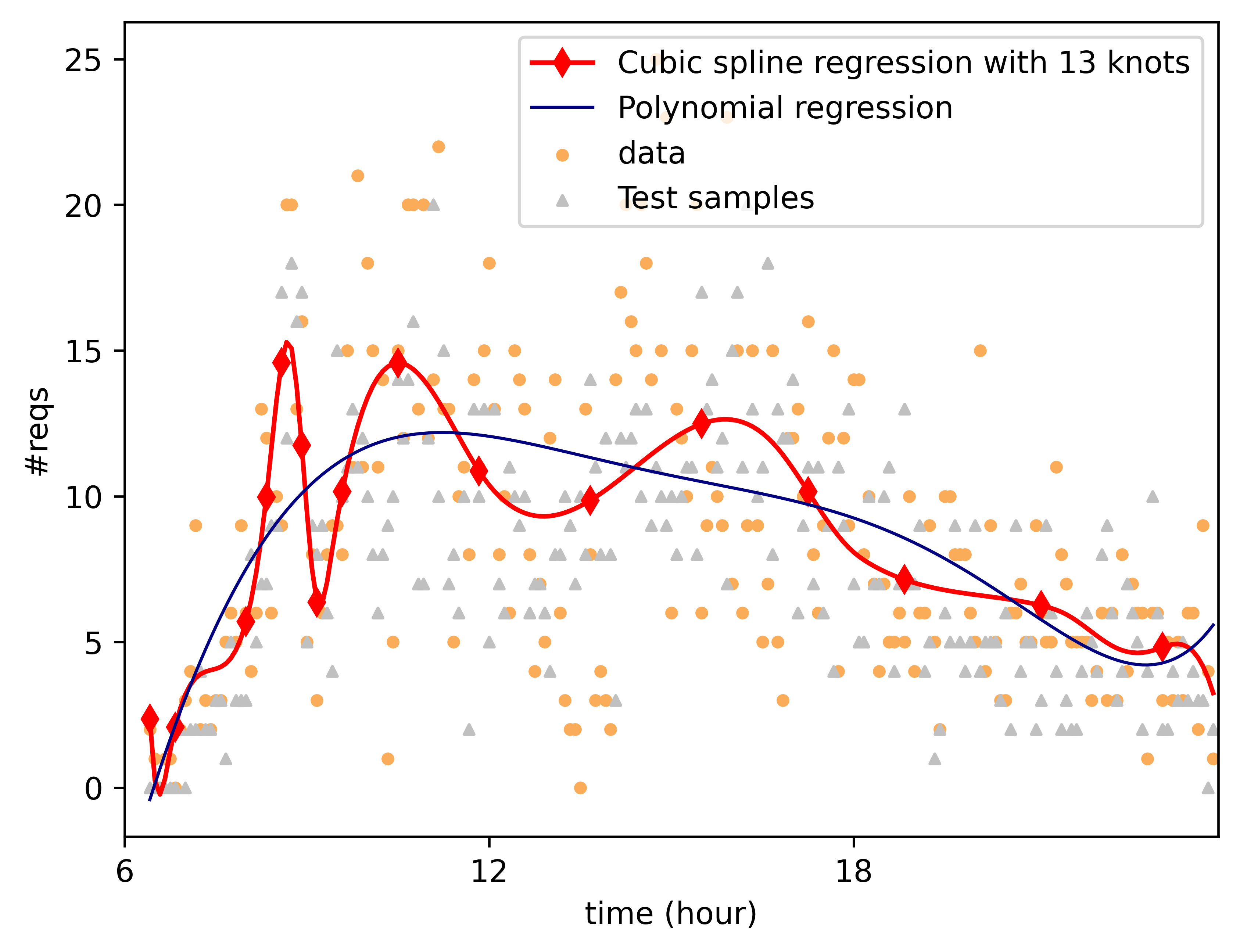}%
		\label{Apendix-figure-regressionIB-poly-Tikhonov}}
	\caption{The piecewise-polynomial regression of the arrival rate function on the IB dataset}
	\label{Apendix-figure-regressionIB-poly}
\end{figure}

\end{appendices}

\bibliography{ref}

\end{document}